\documentclass[conference]{IEEEtran}
\IEEEoverridecommandlockouts
\usepackage{amsmath,amssymb}
\usepackage{amsthm}
\usepackage{float}
\usepackage{algorithm}
\usepackage{algpseudocode}
\usepackage{booktabs}
\usepackage{multirow}
\usepackage{xcolor}
\usepackage{colortbl}
\usepackage{graphicx} 
\usepackage{hyperref}
\usepackage{authblk}

\hypersetup{
  colorlinks=true,
  linkcolor=red,
  citecolor=green!50!black,
  urlcolor=blue
}

\providecommand{\Description}[1]{}

\theoremstyle{plain}
\newtheorem{proposition}{Proposition}
\newtheorem{corollary}{Corollary}
\def\BibTeX{{\rm B\kern-.05em{\sc i\kern-.025em b}\kern-.08em
    T\kern-.1667em\lower.7ex\hbox{E}\kern-.125emX}}
\begin{document}

\title{\fontsize{23}{27}\selectfont
TopoMamba: Topology-Aware Scanning and Fusion for\\
Segmenting Heterogeneous Medical Visual Media}

\author[12]{Fuchen Zheng}
\author[3]{Chengpei Xu}
\author[3]{Long Ma}
\author[6]{Weixuan Li}
\author[6]{Junhua Zhou}
\author[7]{Xuhang Chen}
\author[1]{Weihuang Liu}
\author[4]{\authorcr Haolun Li}
\author[6]{Quanjun Li}
\author[5]{Zhenxi Zhang}
\author[1]{Lei Zhao}
\author[1$\dagger$]{Chi-Man Pun\thanks{$\dagger$ Corresponding Author.}}
\author[2$\dagger$]{Shoujun Zhou}
\affil[1]{University of Macau}
\affil[2]{Shenzhen Institutes of Advanced Technology, Chinese Academy of Sciences}
\affil[3]{Dalian University of Technology}
\affil[4]{Nanjing University of Posts and Telecommunications}
\affil[5]{Hong Kong Polytechnic University}
\affil[6]{Guangdong University of Technology}
\affil[7]{Huizhou University}

\maketitle

\begin{abstract}
Visual state-space models (SSMs) have shown strong potential for medical image segmentation, yet their effectiveness is often limited by two practical issues: axis-biased scan ordering weakens the modeling of oblique and curved structures, and naive multi-branch fusion tends to amplify redundant responses. We present TopoMamba, a topology-aware scan-and-fuse framework for segmenting heterogeneous medical visual media. The method combines a diagonal/anti-diagonal TopoA-Scan branch with the standard Cross-Scan branch to provide complementary structural priors, and introduces ScanCache, a device-aware caching mechanism that amortizes explicit scan-index construction across recurring resolutions. To fuse heterogeneous scan features efficiently, we further propose a lightweight HSIC Gate that regulates branch interaction using a dependence-aware scalar gating rule. We also instantiate a volumetric TopoMamba-3D for practical 3D clinical segmentation. Experiments on Synapse CT, ISIC 2017 dermoscopy, and CVC-ClinicDB endoscopy show that TopoMamba consistently improves segmentation quality over strong CNN, Transformer, and SSM baselines, with particularly clear gains on thin or curved targets such as the pancreas and gallbladder, while maintaining favorable deployment efficiency under dynamic input resolutions. These results suggest that topology-aware scan ordering and lightweight dependence-aware fusion form an effective and practical design for medical multimedia segmentation. The code will be made publicly available.
\end{abstract}

\begin{figure*}[ht]
\centering
\includegraphics[width=1\textwidth]{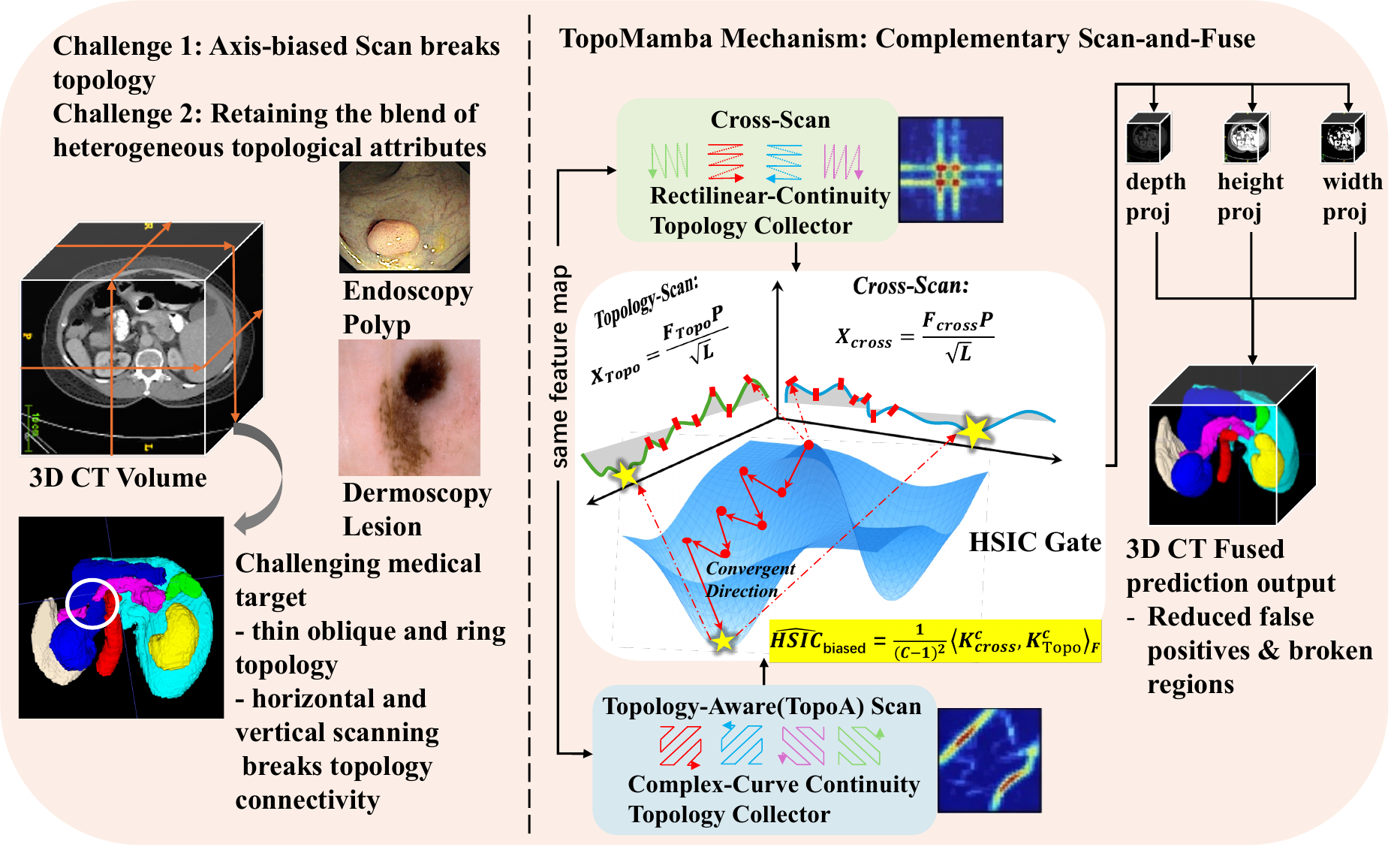}
\caption{Motivation and overview of \textbf{TopoMamba}. Left: axis-biased scanning disrupts non-axial structural continuity in 3D CT, endoscopy, and dermoscopy. Right: Cross-Scan and TopoA-Scan are fused by the HSIC Gate to better preserve non-axial continuity and suppress false positives.}
\label{fig1}
\end{figure*}

\begin{figure*}[ht]
\centering
\includegraphics[width=\linewidth]{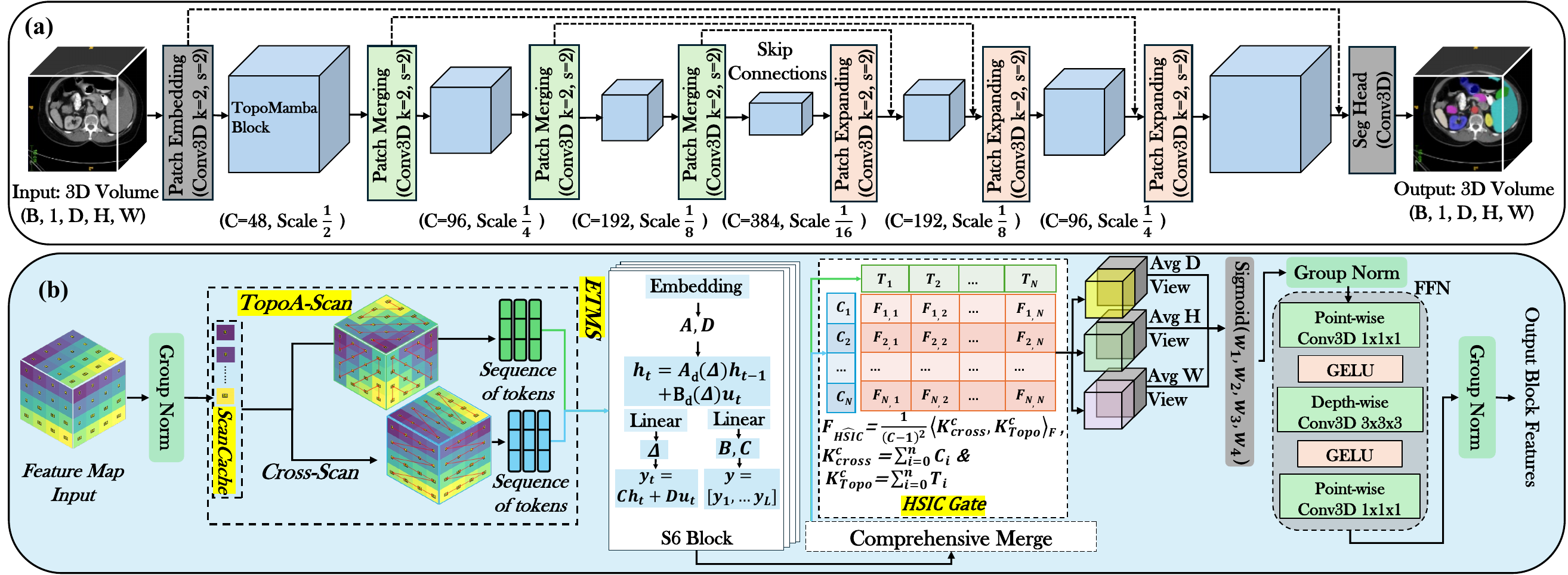}
\caption{Architecture of \textbf{TopoMamba-3D}. (a) 3D U-shaped segmentation network with patch embedding, hierarchical TopoMamba blocks, patch merging and expanding, and skip connections. (b) TopoMamba block with TopoA-Scan and Cross-Scan coupled via ScanCache, state-space sequence modeling, HSIC Gate fusion, view-aware reweighting, and a 3D feed-forward module.}
\label{fig2}
\end{figure*}

\section{Introduction}
\label{sec:intro}

Clinical segmentation systems must handle geometric complexity and deployment pressure at the same time. Abdominal CT media often contain elongated and oblique anatomical structures. Dermoscopic media usually exhibit ambiguous lesion boundaries. Endoscopic media further introduce severe artifacts, such as specular highlights and motion blur~\cite{89medsam_nature2024segment,110MedicalIMGSeg}. Moreover, hospital workflows often revisit fixed resolutions under strict latency constraints. Once deployed, an accurate model that avoids repeatedly reconstructing scan metadata can substantially reduce the time required for surgical planning.

Vision state-space models (SSMs)~\cite{104visionMamba2024vision,105VMamba2024vmamba} have received increasing attention due to their Transformer-like long-range modeling ability and near-linear computational efficiency. Medical SSM variants~\cite{162Umamba2024u,161swin-umamba2024swin,163VM-Unet2024vm} have also shown strong potential. Their main limitation is that most visual SSMs still rely on Cross-Scan style sequence ordering. This axis-aligned traversal is efficient, yet it disrupts diagonal and curved structures, weakening spatial continuity for non-axial anatomy~\cite{139zigzagMamba2024zigma,112snakeConv2023,39consistencyloss}. To address this issue, we introduce topology-aware scanning (TopoA-Scan). It traverses the feature grid along alternating diagonals, thus providing a more suitable implicit structural bias than purely axial ordering.

Recent SSM variants~\cite{170RSMamba2024rsmamba,171plainmamba2024plainmamba} have also started to study complementary feature fusion across different scans. However, direct fusion of multi-branch features usually relies on simple accumulation or heavy fusion heads, which may amplify redundancy. We therefore design a lightweight Hilbert-Schmidt Independence Criterion (HSIC) Gate~\cite{126HSIC2005Gretton_measuring,127HSIC2007Gretton_kernel,128HSIC2012Song_feature}. It exploits projected kernel dependence and adaptively balances Cross-Scan and TopoA-Scan features with only a single learnable scalar. At the system level, we further introduce ScanCache, a device-aware buffer caching mechanism. It converts repeated $O(HW)$ index construction into amortized lookup after the first traversal.

Based on these observations, we develop a topology-aware U-shaped state-space model (SSM), \textbf{TopoMamba}, for heterogeneous medical visual media. Our contributions are summarized as follows:

\begin{itemize}

\item We present \textbf{TopoMamba}, a practical combination of complementary scan ordering, lightweight dependence-aware fusion, and cache-aware inference for medical SSM segmentation.

\item We propose \textbf{ETMS}, an efficient topological multi-scan module built on \textbf{TopoA-Scan} and \textbf{ScanCache}, to introduce non-axial structural bias while preserving scan reuse. For $512\times512$ inputs, \textbf{ScanCache} reduces latency by 47.9\% (72.1→37.6 ms), while the topology-aware scan boosts pancreas Dice from 77.27\% to 79.72\% (+2.45 points).

\item We design a plug-and-play \textbf{HSIC Gate}. It uses only one learnable scalar, avoids heavy fusion heads, and enables stable dependence-aware fusion. We further instantiate a volumetric version, TopoMamba-3D, which preserves the same scan-and-fuse principle in 3D space.

\item TopoMamba is validated on heterogeneous medical visual media datasets, including CT, dermoscopic, and endoscopic images. The results show that it remains efficient under dynamic resolutions and outperforms CNN/Transformer/SSM baselines and remains competitive with SAM-based models on Synapse.

\end{itemize}

\section{Related Work}
\label{sec:related}

\textbf{Medical segmentation backbones.}

Convolutional neural networks (CNNs), such as U-Net~\cite{4unet,13unet++2018}, and vision transformers (ViTs)~\cite{34vit,12swin2021} have long served as the main backbones for medical segmentation. State-space models (SSMs)~\cite{105VMamba2024vmamba,104visionMamba2024vision} provide global contextual modeling with linear complexity. Recent medical SSM variants, including U-Mamba, Swin-UMamba, Neural Memory State Space Models, MPS-Mamba, and DCSS-UNet~\cite{162Umamba2024u,161swin-umamba2024swin,194NeuralMemorySSM2025,196MPSMamba2025,192DCSSUNet2024}, further confirm the value of adapting visual SSMs to anatomy- and modality-specific structure. Our work differs from these backbones by focusing on scan ordering, lightweight dependence-aware fusion, and reusable scan caching within a VMamba-compatible U-shaped design.

\textbf{Scan ordering for visual SSMs.}
Axial attention and criss-cross attention~\cite{113axialAttention2020,138Criss-CrossAttention2019ccnet} decompose 2D attention into 1D operations along the height and width axes. Yet they are limited in modeling oblique continuity. Recent medical SSM studies have explored richer ordering or topology-aware priors through zigzag traversal, rotation-aware scan organization, and wavelet-topology coupling~\cite{139zigzagMamba2024zigma,191RotateToScan2024,190TopoWMamba2025}. Inspired by JPEG zigzag ordering~\cite{120JPEG2000,121JPEG1992pennebaker}, we introduce TopoA-Scan, a locality-preserving diagonal traversal that provides a complementary structural basis to Cross-Scan in SSMs while preserving compatibility with VMamba-style design. In addition, operator-level optimization reduces I/O overhead but often overlooks auxiliary metadata, such as scan indices in SSMs~\cite{136AIICS2020cuda,135AIICS2018nvidiaCaching}. Our ScanCache algorithm promotes this metadata reuse to a first-order deployment design choice.

\textbf{Dependence-aware fusion.} Naive fusion by element-wise summation or direct concatenation lacks dependency regulation and may amplify redundant responses. Recent medical segmentation models have explored multi-scale state-space fusion, channel-adaptive skip fusion, and dual-branch SSM interaction~\cite{193SCFMUNet2025,195DualBranchSSM2025}. Kernel-based dependence measures, such as HSIC~\cite{126HSIC2005Gretton_measuring,127HSIC2007Gretton_kernel,128HSIC2012Song_feature}, provide a stable and non-parametric alternative for modulating branch interaction. Our HSIC Gate differs from heavier fusion heads by using only one learnable scalar to regulate the relative contribution of Cross-Scan and TopoA-Scan.

\textbf{Efficient inference under structured scans.}
Operator-level acceleration has been widely studied, yet auxiliary metadata such as scan indices is often treated as negligible overhead~\cite{136AIICS2020cuda,135AIICS2018nvidiaCaching}. For topology-aware scan styles, this assumption is less reliable, because non-affine reordering is explicit rather than implicit. ScanCache therefore turns index reuse into a first-order optimization. This is particularly important in deployment settings that repeatedly access a small set of resolutions.

\textbf{Topology-aware segmentation beyond scan design.}
We also distinguish our scan-based topology-aware design from explicit topology-preserving segmentation methods~\cite{197TopoSegNet2025}, which impose topology more directly through dedicated losses or topological primitives. In contrast, TopoMamba injects a structural bias through scan ordering and feature fusion. It is therefore topology-aware rather than topology-guaranteeing.

\section{Method}
\label{sec:method}

\subsection{Overview}

TopoMamba is a hierarchical U-Net-like architecture composed of visual state-space (VSS) blocks, as illustrated in Figure~\ref{fig2}. Our core contribution is a topology-aware selective scan module, denoted as \textbf{TopoMamba Block}, which combines an Efficient Topological Multi-Scan (ETMS) module with an HSIC Gate.

Each ETMS block runs two complementary branches: (i) a Cross-Scan branch that preserves compatibility with VMamba-style pretrained weights~\cite{105VMamba2024vmamba}, and (ii) a topology-aware TopoA-Scan branch that better retains oblique, curved, and ring-like structures. ETMS also integrates ScanCache, which registers scan indices as device-aware buffers and reuses them with an LRU policy~\cite{134AIICS2003LRU}.

The two branches are projected into a shared high-dimensional feature distribution space~\cite{150HSIC_projection_johnson1984extensions} and fused through the HSIC Gate. This design exposes differences between branch statistics while allowing the gate to adaptively regulate how much the fused representation relies on Cross-Scan versus TopoA-Scan.

\begin{algorithm}[htb]
\caption{Base Diagonal TopoA Order Generation}
\label{alg:topoa_base_order}
\begin{algorithmic}[1]
\Require Image dimensions $H$, $W$
\Ensure Base diagonal order $\operatorname{diag} \in \mathbb{Z}^{HW}$
\State $\operatorname{diag} \gets \emptyset$
\For{$k \gets 0$ \textbf{to} $H + W - 2$}
    \State $\mathcal{T}_k \gets \emptyset$
    \For{$i \gets 0$ \textbf{to} $H - 1$}
        \State $j \gets k - i$
        \If{$0 \le j < W$}
            \State $\mathcal{T}_k \gets \mathcal{T}_k \cup \{(i, j)\}$
        \EndIf
    \EndFor
    \If{$k \bmod 2 = 1$}
        \State $\mathcal{T}_k \gets \operatorname{reverse}(\mathcal{T}_k)$
    \EndIf
    \ForAll{$(i, j) \in \mathcal{T}_k$}
        \State $\operatorname{diag} \gets \operatorname{diag} \cup \{i \times W + j\}$
    \EndFor
\EndFor
\State \Return $\operatorname{diag}$
\end{algorithmic}
\end{algorithm}

\begin{algorithm}[thb]
\caption{ScanCache Algorithm}
\label{alg:scancache}
\begin{algorithmic}[1]
\Require Input dimensions $H, W$, target device $d$, cache $\mathcal{C}$, cache capacity $M$
\Ensure TopoA index pair $(\mathbf{I}_{\text{fwd}}, \mathbf{I}_{\text{inv}})$

\State $k \leftarrow (H, W, \text{device\_signature}(d))$

\If{$k \in \mathcal{C}$}
    \State $(\mathbf{I}_{\text{fwd}}, \mathbf{I}_{\text{inv}}) \leftarrow \mathcal{C}[k]$
    \If{$\text{device}(\mathbf{I}_{\text{fwd}}) \neq d$}
        \State $(\mathbf{I}_{\text{fwd}}, \mathbf{I}_{\text{inv}}) \leftarrow \text{transfer}\big((\mathbf{I}_{\text{fwd}}, \mathbf{I}_{\text{inv}}), d\big)$
        \State $\mathcal{C}[k] \leftarrow (\mathbf{I}_{\text{fwd}}, \mathbf{I}_{\text{inv}})$
    \EndIf
    \State $\text{update\_LRU}(\mathcal{C}, k)$
    \State \Return $(\mathbf{I}_{\text{fwd}}, \mathbf{I}_{\text{inv}})$
\EndIf

\State $\operatorname{diag} \leftarrow \textsc{BuildBaseDiagonal}(H, W)$
\State $\operatorname{antidiag} \leftarrow \textsc{BuildBaseAntiDiagonal}(H, W)$
\State $\mathbf{I}_{\text{fwd}} \leftarrow [\operatorname{diag}, \operatorname{antidiag}, \operatorname{flip}(\operatorname{diag}), \operatorname{flip}(\operatorname{antidiag})]$
\State Allocate $\mathbf{I}_{\text{inv}} \in \mathbb{Z}^{4 \times HW}$
\For{$m \gets 0$ \textbf{to} $3$}
    \For{$j \gets 0$ \textbf{to} $HW - 1$}
        \State $\mathbf{I}_{\text{inv}}[m, \mathbf{I}_{\text{fwd}}[m, j]] \leftarrow j$
    \EndFor
\EndFor

\State $\text{register\_nonpersistent\_buffer}(\mathcal{C}, k, (\mathbf{I}_{\text{fwd}}, \mathbf{I}_{\text{inv}}))$
\If{$|\mathcal{C}| > M$}
    \State $k_{\text{old}} \leftarrow \text{get\_LRU\_key}(\mathcal{C})$
    \State $\text{evict}(\mathcal{C}, k_{\text{old}})$
\EndIf

\State \Return $(\mathbf{I}_{\text{fwd}}, \mathbf{I}_{\text{inv}})$
\end{algorithmic}
\end{algorithm}

\subsection{The TopoMamba Block}

As shown in Figure~\ref{fig2}(b), the core innovation of TopoMamba lies in the TopoMamba block, which consists of Efficient Topological Multi-Scan (ETMS) Block with ScanCache and the HSIC Gate.

\subsubsection{ Efficient Topological Multi-Scan (ETMS)}
The ETMS module consists mainly of two scanning methods, namely Cross-Scan~\cite{105VMamba2024vmamba} and TopoA-Scan for feature extraction, together with ScanCache, a scan-index reuse algorithm suitable for SSMs.

Cross-Scan Branch: We adopt the Cross-Scan strategy from VMamba~\cite{105VMamba2024vmamba} to capture axis-aligned structure and load pretrained weights. A detailed SSM formulation is given in the supplementary material.
% Appendix~\ref{sec:appendixA.2}.

\textbf{Topology-aware Scan (TopoA-Scan) Branch:} Inspired by ZigzagMamba~\cite{139zigzagMamba2024zigma} and JPEG image compression algorithms~\cite{121JPEG1992pennebaker}, we implement a PyTorch version of the topology-aware scan algorithm for extracting image features in the SSM model. Similar to Cross-Scan, we generate four diagonal scanning patterns for a feature map of shape $(H, W)$. We traverse along diagonals where $i + j = s$ for diagonal index $s \in [0, H+W-2]$. For each diagonal, we alternate the traversal direction:

\begin{equation}
\text{diag}(s) = \begin{cases}
\mathcal{D}(s) & \text{if } s \text{ is even} \\
\text{reverse}(\mathcal{D}(s)) & \text{if } s \text{ is odd}
\end{cases}
\end{equation}
where $\mathcal{D}(s) = \{(i, s-i) \mid 0 \leq i < H, 0 \leq s-i < W\}$.

Similarly, we traverse anti-diagonals where $i - j = s - (W-1)$:

\begin{equation}
\text{antidiag}(s) = \begin{cases}
\mathcal{A}(s) & \text{if } s \text{ is even} \\
\text{reverse}(\mathcal{A}(s)) & \text{if } s \text{ is odd}
\end{cases}
\end{equation}
where $\mathcal{A}(s) = \{(i, W-1-(s-i)) \mid 0 \leq i < H, 0 \leq s-i < W\}$.

Therefore, we obtain the complete forward index set as follows:
\begin{equation}
\begin{aligned}
\mathbf{I}_{\text{fwd}} = [\text{diag}, \text{antidiag},
\text{flip}(\text{diag}), \text{flip}(\text{antidiag})]
\end{aligned}
\end{equation}
where $\mathbf{I}_{\text{fwd}} \in \mathbb{Z}^{4 \times L}$, $L = H \times W$, and each row contains indices in $[0, L-1]$ for flattened feature maps.

Finally, we compute inverse indices to restore the original spatial order after selective scan~\cite{105VMamba2024vmamba}:

\begin{equation}
\begin{aligned}
\mathbf{I}_{\text{inv}}[k, \mathbf{I}_{\text{fwd}}[k, j]] = j, \quad \forall k \in [0,3], j \in [0, L-1]
\end{aligned}
\end{equation}

This ensures that after gathering features with $\mathbf{I}_{\text{fwd}}$ and processing, we can scatter them back using $\mathbf{I}_{\text{inv}}$. Furthermore, we prove that TopoA-Scan provides Spatial Locality Preservation~\cite{122spatiallocality2012space} and complete coverage~\cite{124completecoverage2002jpeg} of the image grid in supplementary material.
% Appendix~\ref{sec:appendixB}.

\textbf{ScanCache:}
SSM scanning~\cite{104visionMamba2024vision,105VMamba2024vmamba} methods recompute index patterns $\mathbf{I}_{\text{fwd}}$ and $\mathbf{I}_{\text{inv}}$ at every forward pass, incurring significant overhead. To address this, we introduce \textbf{ScanCache}, a caching algorithm that registers scanning indices as non-persistent PyTorch buffers managed as part of the model's state. Unlike dictionary-based caching, which stores indices in CPU memory, buffers reside directly on GPU and migrate seamlessly when the model is moved via \texttt{.cuda()}, eliminating costly CPU-GPU transfers.

Given an input feature map $\mathbf{X} \in \mathbb{R}^{B \times C \times H \times W}$, the complete TopoA-Scan with ScanCache proceeds as follows. First, ScanCache verifies whether the pair $(\mathbf{I}_{\text{fwd}}, \mathbf{I}_{\text{inv}})$ already exists. If the pair is found, it is retrieved directly; otherwise, the new scan index pair is registered as a buffer. Meanwhile, a Least Recently Used (LRU) eviction policy~\cite{134AIICS2003LRU} automatically removes stale scan patterns to maintain memory efficiency when cache capacity is exceeded. With buffer-based caching on GPU, repeated index construction is converted into amortized lookup after the first request. To avoid mixing incompatible measurements, we separately report index-service overhead for constructing or retrieving $(\mathbf{I}_{\text{fwd}}, \mathbf{I}_{\text{inv}})$ and end-to-end network latency in Section~\ref{sec:scancache_efficiency}. More details of the ScanCache algorithm are provided in the supplementary material.

\begin{algorithm}[htb]
\caption{HSIC Gate}
\label{alg:hsic_gate}
\begin{algorithmic}[1]
\Require $\mathbf{F}_{\text{cross}}, \mathbf{F}_{\text{TopoA}} \in \mathbb{R}^{B \times C \times L}$; selected projection cap $d_{\text{proj}} = 64$; scale $\alpha = 0.5$; temperature $T = 1.5$; residual weight $\rho = 0.2$
\Ensure $\mathbf{F}_{\text{out}} \in \mathbb{R}^{B \times C \times L}$

\For{$b = 1$ to $B$}
  \State $k \leftarrow \max(8, \min(d_{\text{proj}}, L))$
  \State Retrieve or sample a normalized random projection $P \in \mathbb{R}^{L \times k}$
  \State $X_C \leftarrow \operatorname{norm}_{\ell_2}\!\big((\mathbf{F}_{\text{cross}}[b]P)/\sqrt{L}\big)$
  \State $X_T \leftarrow \operatorname{norm}_{\ell_2}\!\big((\mathbf{F}_{\text{TopoA}}[b]P)/\sqrt{L}\big)$
  \State Compute pairwise squared distances and the median bandwidth $\sigma_{\text{med}}^2$
  \State $K_C[i,j] \leftarrow \exp\!\big(-\|X_C[i,:]-X_C[j,:]\|^2/(2\sigma_{\text{med}}^2)\big)$
  \State $K_T[i,j] \leftarrow \exp\!\big(-\|X_T[i,:]-X_T[j,:]\|^2/(2\sigma_{\text{med}}^2)\big)$
  \State Center kernels by row mean, column mean, and global mean
  \State Denote the centered kernels by $K_C^c$ and $K_T^c$
  \State $\widehat{\mathrm{HSIC}} \leftarrow \frac{1}{(C-1)^2} \sum_{i,j} (K_C^c)_{ij}(K_T^c)_{ij}$
  \State $w \leftarrow \sigma(\alpha \widehat{\mathrm{HSIC}}/T)$
  \State $\mathbf{F}_{\text{fuse}} \leftarrow w\,\mathbf{F}_{\text{TopoA}}[b] + (1-w)\,\mathbf{F}_{\text{cross}}[b]$
  \State $\mathbf{F}_{\text{out}}[b] \leftarrow (1-\rho)\,\mathbf{F}_{\text{fuse}} + \rho\,\mathbf{F}_{\text{TopoA}}[b]$
\EndFor

\State \Return $\mathbf{F}_{\text{out}}$
\end{algorithmic}
\end{algorithm}

% \begin{algorithm}[H]
% \caption{HSIC Gate}
% \label{alg:hsic_gate}
% \begin{algorithmic}[1]
% \Require $\mathbf{F}_{\text{cross}}, \mathbf{F}_{\text{TopoA}} \in \mathbb{R}^{B \times C \times L}$; projection cap $d_{\text{proj}}$; scale $\alpha$; temperature $T{=}1.5$; residual $\rho{=}0.2$
% \Ensure $\mathbf{F}_{\text{out}} \in \mathbb{R}^{B \times C \times L}$
% \For{$b=1$ to $B$}
%   \State $k \gets \max(8,\min(d_{\text{proj}},L))$
%   \State Sample $P \in \mathbb{R}^{L \times k}$; column-normalize
%   \State $X_C \gets \operatorname{norm}_{\ell_2}\!\big((\mathbf{F}_{\text{cross}}[b]P)/\sqrt{L}\big)$
%   \State $X_T \gets \operatorname{norm}_{\ell_2}\!\big((\mathbf{F}_{\text{TopoA}}[b]P)/\sqrt{L}\big)$
%   \State Compute pairwise squared distances and median bandwidth $\sigma^2_{\text{med}}$
%   \State $K_C[i,j] \gets \exp\!\big(-\|X_C[i,:]-X_C[j,:]\|^2/(2\sigma^2_{\text{med}})\big)$
%   \State $K_T[i,j] \gets \exp\!\big(-\|X_T[i,:]-X_T[j,:]\|^2/(2\sigma^2_{\text{med}})\big)$
%   \State Center by means: $K^c \gets K - r - c + m$
%   \State $\widehat{\mathrm{HSIC}} \gets \frac{1}{(C-1)^2}\sum_{i,j}(K_C^{c})_{ij}(K_{T}^{c})_{ij}$
%   \State $w \gets \sigma(\alpha \widehat{\mathrm{HSIC}}/T)$
%   \State $\mathbf{F}_{\text{out}}[b] \gets (1-\rho)\,(w\,\mathbf{F}_{\text{TopoA}}[b]+(1-w)\,\mathbf{F}_{\text{cross}}[b])+\rho\,\mathbf{F}_{\text{TopoA}}[b]$
% \EndFor
% \State \Return $\mathbf{F}_{\text{out}}$
% \end{algorithmic}
% \end{algorithm}

\begin{figure}[ht]
\centerline{\includegraphics[width=0.97\columnwidth]{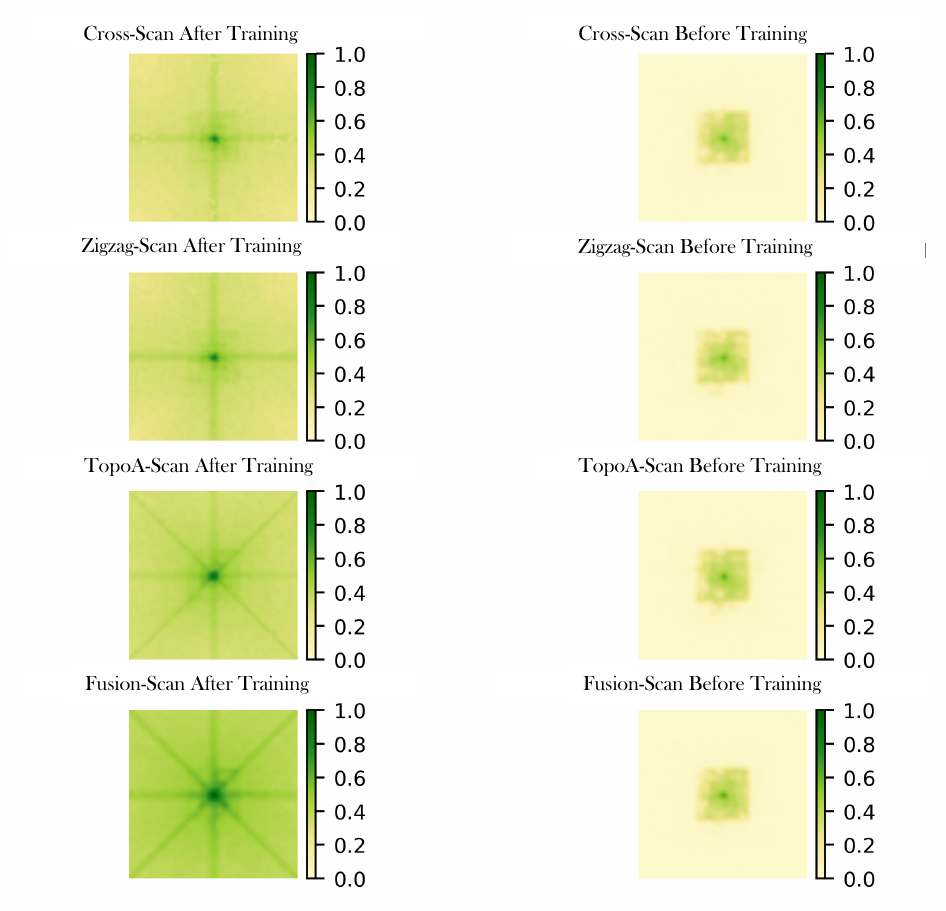}}
\Description{Effective receptive field comparison across scan variants, illustrating that the fused TopoMamba design covers broader and more structured regions than individual scan branches.}
\caption{Effective receptive fields (ERFs)~\cite{168ERF2016understanding} before and after training, averaged over 300 slices. ERFs are computed from unit input perturbations and normalized output-gradient energy. Fusion-Scan denotes Cross-Scan fused with TopoA-Scan via the HSIC Gate.}
\label{ERF} 
\end{figure}

% ----------HSIC Gate A
\subsubsection{Hilbert-Schmidt Independence Criterion (HSIC) Gate}

After obtaining features from both branches, we need an adaptive fusion mechanism to combine them effectively. Simply averaging or concatenating dual-branch features is suboptimal because different image regions may benefit from different scanning patterns. For instance, Cross-Scan is effective for axis-aligned boundaries and elongated anatomical extents, whereas TopoA-Scan is better suited to oblique, curved, and ring-like patterns frequently observed in lesions, polyps, and thin organs. Therefore, we introduce the \textbf{HSIC Gate}, which efficiently measures statistical dependence between feature maps $\mathbf{F}_{\text{cross}}, \mathbf{F}_{\text{TopoA}} \in \mathbb{R}^{B \times C \times L}$ via the Hilbert-Schmidt Independence Criterion~\cite{126HSIC2005Gretton_measuring,127HSIC2007Gretton_kernel}. For brevity, both TopoA-Scan and Cross-Scan expressions are unified using the `scan' notation.

We first apply Johnson-Lindenstrauss random projection~\cite{150HSIC_projection_johnson1984extensions,177JL_Proj2003database} to reduce computational overhead and stabilize kernel computation:
\begin{equation}
\mathbf{X}_{\text{scan}} = \frac{\mathbf{F}_{\text{scan}} \mathbf{P}}{\sqrt{L}} \in \mathbb{R}^{B \times C \times k}
\end{equation}
where $\mathbf{P}$ is a normalized random projection matrix and $\mathbf{X}_{\text{scan}}$ denotes projected features preserving pairwise distances with high probability.

We then compute RBF kernels~\cite{178RBF_scholkopf2002learning,179RBF_shawe2004kernel} to capture nonlinear dependencies:
\begin{equation}
    \mathbf{K}_{\text{scan}}[i,j]=\exp\!\left(-\frac{\|\mathbf{X}_{\text{scan}}[i,:]-\mathbf{X}_{\text{scan}}[j,:]\|_2^2}{2\sigma^2}\right).
\end{equation}
where $\sigma$ is the kernel bandwidth~\cite{180RBF_camps2009kernel}.
% determined by the median heuristic per batch.

Next, we center the kernel matrices to remove mean effects~\cite{126HSIC2005Gretton_measuring}:
\begin{equation}
\tilde{\mathbf{K}}_{\text{scan}} = \mathbf{H}\mathbf{K}_{\text{scan}}\mathbf{H}
\end{equation}
where $\mathbf{H} = \mathbf{I}_C - \frac{1}{C}\mathbf{1}_C\mathbf{1}_C^T$ is the channel-wise centering matrix~\cite{128HSIC2012Song_feature}. 
% Concretely, HSIC is computed per sample (per b in the batch) on channel–channel Gram matrices of size $C×C$.

The HSIC value is computed as the normalized Frobenius inner product~\cite{126HSIC2005Gretton_measuring}:
\begin{equation}
\text{HSIC}(\mathbf{F}_{\text{cross}}, \mathbf{F}_{\text{TopoA}}) = \frac{\langle \tilde{\mathbf{K}}_{\text{cross}}, \tilde{\mathbf{K}}_{\text{TopoA}} \rangle_F}{(C-1)^2}
\end{equation}
where $\langle \mathbf{A}, \mathbf{B} \rangle_F = \sum_{i,j} A_{ij}B_{ij}$ denotes the Frobenius inner product. Higher HSIC values indicate stronger statistical dependence between the two branches.

Subsequently, we convert HSIC into an adaptive gate weight through a learnable sigmoid function $w=\sigma(\alpha \cdot \text{HSIC}/T)$, where $\alpha$ is a trainable parameter initialized to $0.5$ and $T$ is a temperature factor that smooths the gating sharpness. Under this definition, larger HSIC increases the TopoA-Scan contribution inside the convex fusion term, whereas smaller HSIC relatively increases the Cross-Scan contribution. The final fusion is performed as $\mathbf{F}_{\text{out}}=(1-\rho)\,\big(w \cdot \mathbf{F}_{\text{TopoA}} + (1-w) \cdot \mathbf{F}_{\text{cross}}\big)+\rho\,\mathbf{F}_{\text{TopoA}}$, where $\rho$ provides a residual shortcut for stability. The proposed gate is therefore intentionally TopoA-biased rather than perfectly symmetric: HSIC modulates how strongly the model trusts TopoA-Scan relative to Cross-Scan, while the residual path preserves a minimum TopoA-Scan contribution even when the measured dependence is low.

% \subsection{Volumetric TopoMamba-3D}
% To address the common concern that slice-wise evidence is insufficient for clinical CT segmentation, we also instantiate a volumetric TopoMamba-3D model. Following the implementation in \texttt{models/TopoMamba\_3D.py}, the 3D variant preserves the same encoder-decoder backbone while replacing 2D branches with an axial 3D branch and a topology-aware diagonal 3D branch over $(D,H,W)$. ScanCache is extended to cache 3D scan orders per device and spatial shape, and the HSIC Gate is computed over multiple 3D views, namely the full volume and depth-, height-, and width-aggregated projections. This keeps the main scan-and-fuse principle unchanged while making the strongest Synapse result genuinely volumetric Topology.

\begin{table*}[t]
  \caption{Results on Synapse multi-organ CT. Metrics are Dice similarity coefficient (\%) and HD95 (mm).}
  \label{tab:synapse_results_acm}
  % \scriptsize
  \small
  \setlength{\tabcolsep}{2pt}
  \begin{tabular}{llcccccccccccc}
    \toprule
    Backbone & Method & Param(M) & Spleen & R.Kid & L.Kid & Gall. & Liver & Stom. & Aorta & Panc. & Avg & HD95(mm) \\
    \midrule
    Conv & nnUNet~\cite{157nnUnet2021nature}(Nature’21) & 16.0 & 94.12 & 89.41 & 91.54 & 70.12 & 94.77 & 82.24 & 87.79 & 68.12 & 84.76 & 14.81 \\
    Diffusion & MedSegDiff~\cite{159medsegdiff2024medsegdiff}(PMLR’24) & 34.8 & 94.21 & 92.13 & 93.39 & 89.75 & 94.56 & 86.56 & 89.34 & 78.64 & 89.82 & 4.87 \\
    Transformer & Swin-UNETR~\cite{49swinUNETR}(CVPR’22) & 
    138.0 & 93.34 & 89.93 & 91.88 & 79.49 & 95.72 & 87.10 & 89.38 & 75.11 & 87.74 & 8.18 \\  
    Transformer & SAM (GT Box)~\cite{87sam_meta2023segment}(ICCV’23) & 636.0 & 94.55 & 89.76 & 90.95 & 82.97 & 95.74 & 90.14 & 90.15 & 76.33 & 88.82 & 3.27 \\
    Transformer & MedSAM (GT Box)~\cite{89medsam_nature2024segment}(Nature’24) & 636.0 
    & 95.88 & 90.71 & 91.40 & 90.82 & 92.44 & 90.70 & 92.47 & 76.87 & 90.16 & 2.41 \\
    Transformer & H-SAM (GT Box)~\cite{169H-SAM2024unleashing}(CVPR’24) & 112.3 & 97.10 & 93.60 & 94.30 & 79.40 & 97.50 & 92.10 & 89.20 & 79.40 & 90.32 & 4.81 \\
    SSM & VM-UNet~\cite{163VM-Unet2024vm}(ACM TOMM’24) & 31.0 & 89.31 & 79.81 & 84.78 & 63.28 & 95.02 & 80.40 & 87.68 & 59.66 & 79.99 & 21.20 \\
    SSM & Swin-UMamba~\cite{161swin-umamba2024swin}(MICCAI’24) & 27.6 & 90.16 & 80.75 & 86.85 & 64.28 & 95.00 & 77.17 & 88.13 & 57.70 & 80.01 & 16.42 \\
    \rowcolor{gray!10} SSM & TopoMamba-2D & 23.9 & 95.85 & 90.58 & 92.36 & 78.96 & 96.63 & 92.61 & 92.59 & 81.41 & 90.12 & 4.63 \\
    \rowcolor{gray!10} SSM & TopoMamba-3D & 27.2 & 96.93 & 91.42 & 93.11 & 82.25 & 96.71 & 93.56 & 94.13 & 82.11 & 91.28 & 3.08 \\
    \bottomrule
  \end{tabular}
\end{table*}

\begin{table}[t]
  \caption{Results on the ISIC 2017 dermoscopy and CVC-ClinicDB endoscopy benchmarks.}
  \label{tab:comparison_results_acm}
  \centering
  \footnotesize
  % \small
  \setlength{\tabcolsep}{5pt}
  \resizebox{\linewidth}{!}{
  \begin{tabular}{lccccccc}
    \toprule
    \multirow{2}{*}{Method} & \multirow{2}{*}{Param} & \multicolumn{3}{c}{ISIC 2017} & \multicolumn{3}{c}{CVC-ClinicDB} \\
    \cmidrule(lr){3-5} \cmidrule(lr){6-8}
    & & DSC & mIoU & Sp & DSC & mIoU & Sp \\
    \midrule
    UNet~\cite{4unet} & 31.0 & 84.41 & 73.91 & 95.18 & 86.12 & 78.02 & 95.46 \\
    SwinUNet~\cite{21swinUnet2022} & 27.2 & 85.83 & 78.53 & 95.72 & 86.67 & 79.21 & 95.67 \\
    Mamba-UNet~\cite{160MambaUnet2024mamba} & 19.1 & 88.57 & 78.02 & 96.01 & 89.40 & 79.46 & 96.24 \\
    U-Mamba~\cite{162Umamba2024u} & 23.42 & 88.65 & 79.03 & 95.89 & 88.74 & 79.08 & 96.30 \\
    Swin-UMamba~\cite{161swin-umamba2024swin} & 27.6 & 89.60 & 80.68 & 96.77 & 90.50 & 80.40 & 97.33 \\
    VM-UNet~\cite{163VM-Unet2024vm} & 31.0 & 89.03 & 80.23 & 97.58 & 89.84 & 79.62 & 96.36 \\
    \rowcolor{gray!20} TopoMamba-2D & 23.9 & \textbf{90.07} & \textbf{81.13} & \textbf{98.72} & \textbf{91.05} & \textbf{82.37} & \textbf{98.91} \\
    \bottomrule
  \end{tabular}
  }
\end{table}

\section{Experiments}
\label{sec:experiments}

Our experiments are designed to answer three questions central to topology-aware state-space modeling for multimedia segmentation: (1) whether the proposed scan strategy benefits structures that are not well captured by purely horizontal or vertical traversals; (2) whether the same scan-and-fusion design generalizes across substantially different visual media, from CT to dermoscopy and endoscopy; and (3) whether the design remains computationally practical when explicit scan-index construction is taken into account.

\subsection{Experimental Protocol}

We evaluate TopoMamba on three segmentation benchmarks: Synapse multi-organ CT, ISIC 2017 lesion dermoscopy, and CVC-ClinicDB polyp endoscopy~\cite{45synapse,164ISIC2017skin2018skin,165CVC-ClinicDB}. These datasets present complementary challenges. Synapse contains complex anatomy and weak organ boundaries; ISIC includes irregular lesion contours, hair occlusion, and illumination variation; and CVC features small targets, motion blur, specular highlights, and strong background texture. Consistent improvements across these settings would indicate that the proposed design is not tied to a single visual modality.

For Synapse, we adopt the widely used 18/12 train-test split and report Dice and HD95. Dice measures region overlap, whereas HD95 reflects large boundary deviations, although neither directly characterizes topology preservation. For ISIC and CVC, we report DSC, mIoU, and specificity, since suppressing false positives is particularly important in both datasets. Unless otherwise stated, all TopoMamba models are trained in PyTorch using AdamW, a learning rate of $3\times10^{-4}$, cosine decay, three random seeds, and an input size of $512^2$. Synapse serves as the primary benchmark for evaluating structural continuity, while ISIC and CVC test cross-media generalization; ScanCache is included to assess deployment efficiency when scan indices must be constructed explicitly.

% \begin{figure}[ht]
% \centerline{\includegraphics[width=0.97\columnwidth]{figs/figures/attention_pattern_plot.png}}
% \vspace{-0.7em}
% \caption{Deprecated draft caption kept only for archival reference.}
% \label{fig:attention_heatmap}
% \end{figure}

\begin{figure}[t]
\centering
\includegraphics[width=\columnwidth]{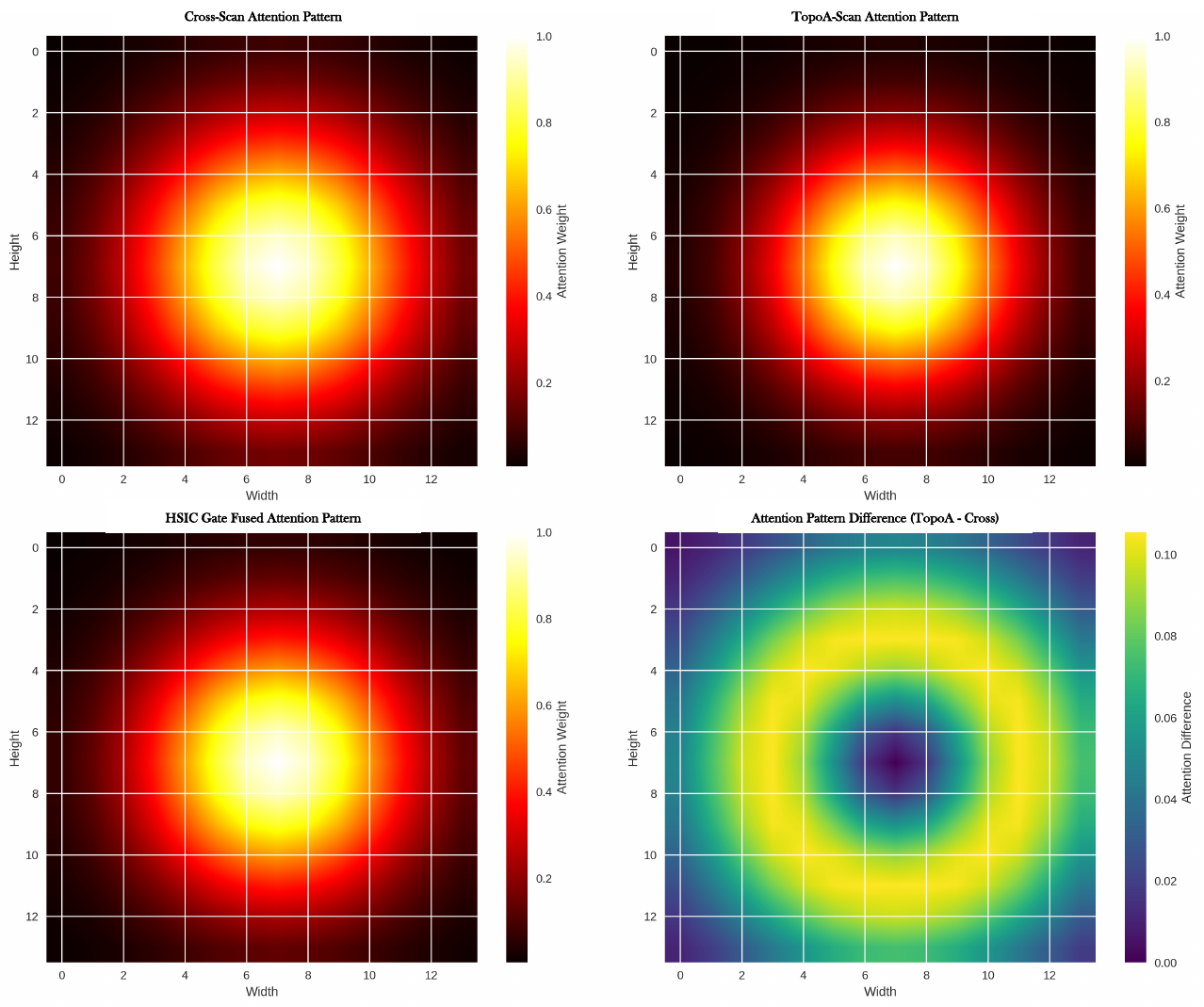}
\caption{Qualitative ETMS attention comparison. Both branches share similar global response centers but differ in directional emphasis: Cross-Scan is more axis-aligned, whereas TopoA-Scan responds more strongly to oblique and curved relations. This figure is qualitative only; quantitative evidence is provided by the controlled ablations and topology-oriented analysis.}
\label{fig:attention_heatmap}
\end{figure}

\begin{figure*}[ht]
\centerline{\includegraphics[width=0.85\textwidth]{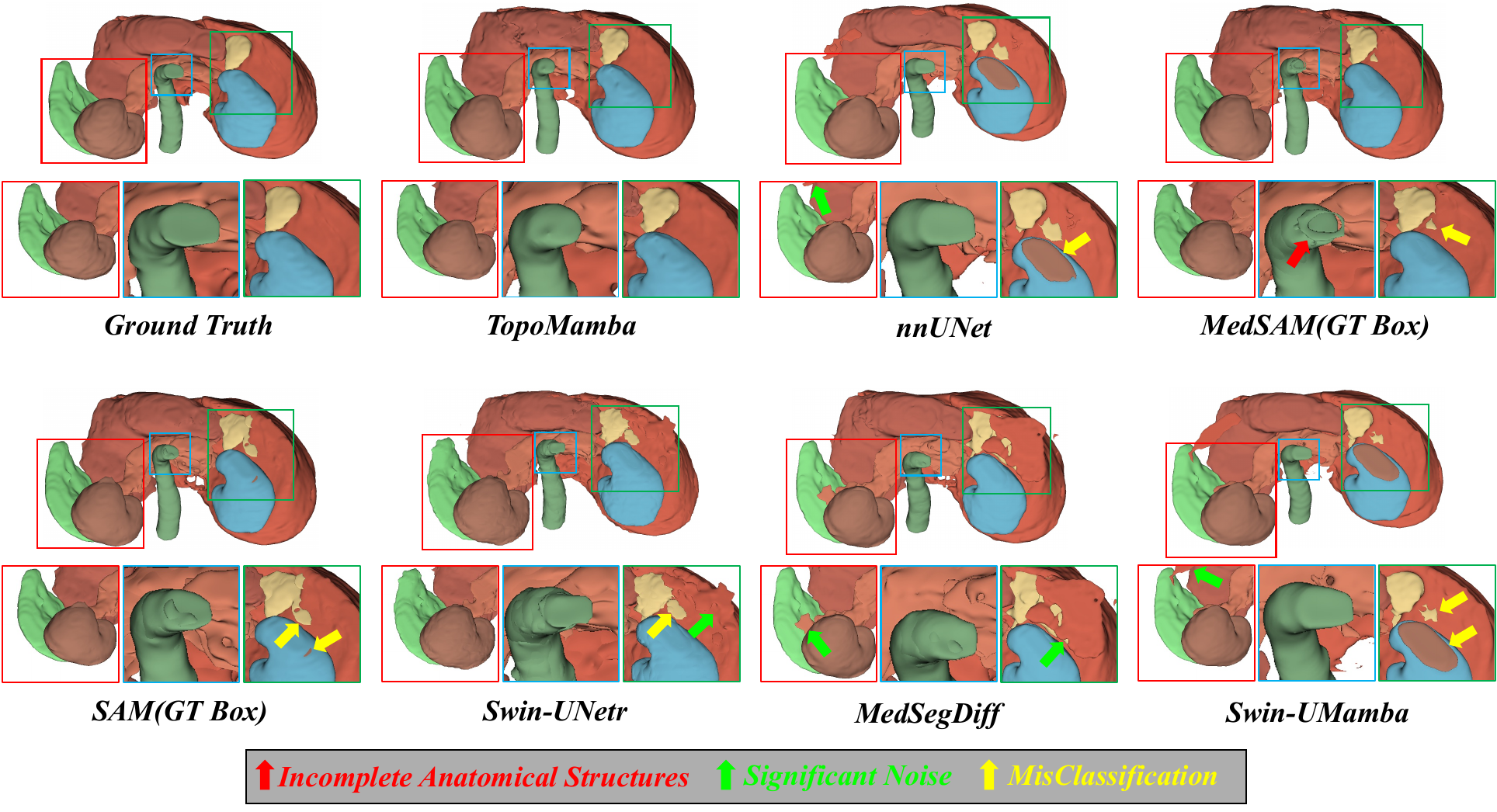}}
\Description{Qualitative Synapse comparisons highlighting failure modes such as incomplete structures, noise, and misclassification across multiple segmentation baselines and TopoMamba.}
\caption{Representative Synapse test cases.}
\label{fig:synapse_vis}
\end{figure*}

% \begin{figure*}[ht]
% \centering
% \begingroup
% \setlength{\unitlength}{0.0098\textwidth}
% \begin{picture}(100,54.25)
% \put(0,0){\includegraphics[width=0.98\textwidth]{figs/Synapse_Vis.pdf}}
% \put(37.8,54.6){\makebox(0,0)[c]{\colorbox{white}{\rule{0pt}{8.0mm}\hspace{34mm}}}}
% \put(37.8,54.8){\makebox(0,0)[c]{\colorbox{white}{\scriptsize\itshape TopoMamba-T}}}
% \end{picture}
% \endgroup
% \Description{Qualitative Synapse comparisons highlighting failure modes such as incomplete structures, noise, and misclassification across multiple segmentation baselines and TopoMamba.}
% \caption{Segmentation results on Synapse Multi-organ dataset.}
% \label{fig:synapse_vis}
% \end{figure*}

\subsection{Synapse Main Benchmark}

Table~\ref{tab:synapse_results_acm} reports the main results on Synapse, where the limitations of purely axis-aligned scanning are especially evident. Organs such as the pancreas and gallbladder are small, thin, curved, and often weakly contrasted, making them sensitive to how long-range spatial dependencies are serialized. TopoMamba-2D achieves 90.12\% average Dice with 23.9M parameters, substantially improving over VM-UNet, which attains 79.99\% with 31.0M parameters. TopoMamba-3D further reaches 91.28\%, indicating that the same design principle also remains effective in the 3D setting. Since Table~\ref{tab:synapse_results_acm} compares end-to-end systems, these results validate the overall architecture rather than isolating scan order alone.

The gains are most pronounced on the most challenging organs. On the pancreas, TopoMamba-2D improves Dice from 59.66\% to 81.41\%, and TopoMamba-3D further reaches 82.11\%. On the gallbladder, the corresponding scores are 78.96\% and 82.25\%. By contrast, gains on larger and anatomically simpler organs such as the liver and spleen are smaller, which is consistent with the intended role of the method in preserving non-axial structural continuity. Prompt-based models remain competitive on boundary quality: MedSAM achieves the best HD95 at 2.41 mm, compared with 4.63 mm for TopoMamba-2D and 3.08 mm for TopoMamba-3D. Nevertheless, TopoMamba delivers substantially stronger overlap than conventional SSM baselines without requiring external prompts. It also surpasses Swin-UNETR, which reaches 87.74\% Dice with 138.0M parameters, while using far fewer parameters.

\begin{table}[t]
  \caption{Topology-oriented evaluation on Synapse. Lower CCE/HCE and higher ETM indicate better topology preservation.}
  \label{tab:topology_results_acm}
  \centering
  \small
  \setlength{\tabcolsep}{6pt}
  \begin{tabular}{lccc}
    \toprule
    Method & CCE$\downarrow$ & HCE$\downarrow$ & ETM(\%)$\uparrow$ \\
    \midrule
    VM-UNet~\cite{163VM-Unet2024vm} & 4.86 & 0.75 & 32.2 \\
    Swin-UMamba~\cite{161swin-umamba2024swin} & 4.79 & 0.68 & 34.6 \\
    H-SAM (GT Box)~\cite{169H-SAM2024unleashing} & 3.41 & 0.54 & 37.6 \\
    MedSegDiff~\cite{159medsegdiff2024medsegdiff} & 3.29 & 0.52 & 38.1 \\
    \rowcolor{gray!15} TopoMamba-2D & 2.89 & 0.48 & 41.8 \\
    \rowcolor{gray!15} TopoMamba-3D & 2.18 & 0.44 & 49.7 \\
    \rowcolor{gray!15} +Topology-Aware Loss~\cite{198topoLoss2023cvpr} & 1.54 & 0.38 & 58.9 \\
    % \rowcolor{gray!15} +nnUNet-style pre/post-processing~\cite{157nnUnet2021nature} & \textbf{0.93} & \textbf{0.29} & \textbf{71.9} \\
    \bottomrule
  \end{tabular}
\end{table}

% \begin{table}[t]
%   \caption{Topology-oriented analysis on Synapse. Lower CCE/HCE is better; higher ETM is better.}
%   \label{tab:topology_results_acm}
%   \centering
%   \footnotesize
%   \setlength{\tabcolsep}{6pt}
%   \begin{tabular}{lccc}
%     \toprule
%     Method & CCE$\downarrow$ & HCE$\downarrow$ & ETM(\%)$\uparrow$ \\
%     \midrule
%     VM-UNet~\cite{163VM-Unet2024vm} & 4.86 & 0.75 & 32.2 \\
%     Swin-UMamba~\cite{161swin-umamba2024swin} & 4.79 & 0.68 & 34.6 \\
%     H-SAM (GT Box)~\cite{169H-SAM2024unleashing} & 3.41 & 0.54 & 57.6 \\   MedSegDiff~\cite{159medsegdiff2024medsegdiff}& 2.79 & 0.48 & 60.1 \\
%     TopoMamba-2D & 3.02 & 0.41 & 65.8 \\
%     TopoMamba-3D & 2.58 & 0.35 & 68.7 \\
%     \rowcolor{gray!15} TopoMamba-3D +Topology-Aware Loss~\cite{198topoLoss2023cvpr} & \textbf{1.47} & \textbf{0.21} & \textbf{73.9} \\
%     \bottomrule
%   \end{tabular}
% \end{table}

\subsection{Topology-oriented Analysis}

Because Dice and HD95 do not explicitly quantify topology, we further report three mask-level topology metrics on Synapse: component count error (CCE), hole count error (HCE), and exact topology match (ETM). Specifically, $\mathrm{ETM}=1$ only when both counts exactly match the ground truth.

As shown in Table~\ref{tab:topology_results_acm}, TopoMamba consistently improves topology-oriented quality over strong SSM baselines. TopoMamba-2D reduces CCE/HCE to 2.89/0.48 and raises ETM to 41.8\%, while TopoMamba-3D further improves these values to 2.18/0.44 and 49.7\%, respectively. When combined with the topology-aware loss, the scores are further strengthened to 1.54/0.38 and 58.9\%. These results are consistent with the improvements observed in Dice and HD95, especially on thin and curved organs. TopoMamba is topology-aware rather than topology-guaranteeing, in that it introduces a structural bias toward better continuity without imposing hard topological constraints.

\subsection{Cross-Dataset Transfer}

Table~\ref{tab:comparison_results_acm} evaluates whether the proposed design remains effective under substantial modality shifts. The answer is affirmative. On ISIC 2017, TopoMamba-2D improves DSC from 89.03\% to 90.07\%, mIoU from 80.23\% to 81.13\%, and specificity from 97.58\% to 98.72\% relative to VM-UNet. On CVC-ClinicDB, the improvements are from 89.84\% to 91.05\% in DSC, from 79.62\% to 82.37\% in mIoU, and from 96.36\% to 98.91\% in specificity. Although the absolute gains are smaller than those on Synapse, they are consistent across all reported metrics.

The specificity improvements are particularly important in these two settings. In dermoscopy, hair occlusion and illumination variation often induce false positives; in endoscopy, specular highlights and surrounding textures may produce similar errors. The results therefore suggest that TopoMamba improves foreground-background discrimination rather than merely enlarging the predicted masks. The gain is more pronounced on CVC-ClinicDB than on ISIC 2017, which is also consistent with the fact that polyps frequently exhibit curved boundaries and stronger local distractors. Overall, the same scan-and-fusion design transfers well across CT, dermoscopy, and endoscopy.

\begin{table}[t]
  \caption{Controlled Synapse ablation under the protocol of Table~\ref{tab:synapse_results_acm}; only scan ordering and fusion vary.}
  \label{tab:scan_components_acm}
  \centering
  \normalsize
  \setlength{\tabcolsep}{5pt}
  \resizebox{\columnwidth}{!}{%
  \begin{tabular}{lccccccccc}
    \toprule
    Configuration & Spl. & R.Kid & L.Kid & Gall. & Liv. & Sto. & Aor. & Pan. & Avg \\
    \midrule
    Cross-Scan only & 94.63 & 89.18 & 90.81 & 76.94 & 96.15 & 91.24 & 91.12 & 77.27 & 88.42 \\
    TopoA-Scan only & 94.14 & 89.10 & 90.33 & 79.08 & 96.19 & 91.47 & 91.86 & 79.72 & 88.98 \\
    Cross-Scan+TopoA-Scan+Add & 88.48 & 85.01 & 86.66 & 62.47 & 86.31 & 81.95 & 82.08 & 69.28 & 80.28 \\
    \rowcolor{gray!15} Cross-Scan+TopoA-Scan+HSIC Gate & \textbf{95.85} & \textbf{90.58} & \textbf{92.36} & \textbf{78.96} & \textbf{96.63} & \textbf{92.61} & \textbf{92.59} & \textbf{81.41} & \textbf{90.12} \\
    \bottomrule
  \end{tabular}}
\end{table}

\subsection{Scan Strategy Ablation}

Table~\ref{tab:scan_components_acm} provides a controlled ablation under the same protocol as Table~\ref{tab:synapse_results_acm}, where only scan ordering and fusion are varied. Among the single-branch settings, TopoA-Scan only achieves 88.98\% average Dice, outperforming Cross-Scan only at 88.42\%. The advantage is most visible on the gallbladder and pancreas, suggesting that TopoA-Scan is better suited to preserving slanted and curved anatomical continuity. At the same time, Cross-Scan remains slightly stronger on several organs, indicating that no single traversal order is uniformly optimal.

The dual-branch results further show that complementarity must be paired with an appropriate fusion mechanism. Simply combining Cross-Scan and TopoA-Scan with element-wise addition yields only 80.28\% average Dice, which is markedly worse than either single-branch setting. By contrast, replacing naive addition with HSIC Gate raises performance to 90.12\%, the best result in the table. This controlled study therefore supports the central claim of the paper: the gain arises from combining complementary scan orders with dependence-aware fusion, rather than from increasing branch count alone.

\begin{table}[t]
  \caption{Fusion comparison within the same Cross-Scan+TopoA-Scan backbone under the protocol of Table~\ref{tab:scan_components_acm}. BFTT3D-, I2P-MAE-, and DTRG-style baselines are reimplemented as lightweight interaction rules rather than task-specific medical segmentation heads.}
  \label{tab:fusion_results_acm}
  \centering
  \normalsize
  \setlength{\tabcolsep}{3pt}
  \resizebox{\columnwidth}{!}{%
  \begin{tabular}{lccccccccc}
    \toprule
    Fusion & Spl. & R.Kid & L.Kid & Gall. & Liv. & Sto. & Aor. & Pan. & Avg \\
    \midrule
    Element-wise Add & 88.48 & 85.01 & 86.66 & 62.47 & 86.31 & 81.95 & 82.08 & 69.28 & 80.28\\
    Concat + Conv & 91.71 & 87.07 & 88.13 & 75.14 & 93.33 & 88.97 & 89.01 & 76.90 & 86.28 \\
    SE channel gate~\cite{33SEnet} & 92.56 & 87.09 & 88.78 & 75.91 & 93.38 & 89.07 & 89.15 & 77.01 & 86.62 \\
    BFTT3D-style fusion~\cite{184BFTT3D_CVPR2024backpropagation} & 95.51 & 90.04 & 91.71 & 78.73 & 96.34 & 91.99 & 92.11 & 79.92 & 89.54 \\
    I2P-MAE-style fusion~\cite{183I2P-MAE_CVPR2023learning} & 95.60 & 90.16 & 91.84 & 78.12 & 96.41 & 92.15 & 92.20 & 79.56 & 89.50 \\
    DTRG-style fusion~\cite{185DTRG_TIP2022convolutional} & 95.47 & 89.98 & 91.62 & 78.58 & 96.37 & 91.90 & 92.07 & 79.69 & 89.46 \\
    \rowcolor{gray!15} HSIC Gate & \textbf{95.85} & \textbf{90.58} & \textbf{92.36} & \textbf{78.96} & \textbf{96.63} & \textbf{92.61} & \textbf{92.59} & \textbf{81.41} & \textbf{90.12} \\
    \bottomrule
  \end{tabular}}
\end{table}

\subsection{Fusion Mechanism Analysis}

Table~\ref{tab:fusion_results_acm} compares several fusion modules under the same Cross-Scan+TopoA-Scan backbone. The BFTT3D-style~\cite{184BFTT3D_CVPR2024backpropagation}, I2P-MAE-style~\cite{183I2P-MAE_CVPR2023learning}, and DTRG-style~\cite{185DTRG_TIP2022convolutional} baselines are included not for their original medical-segmentation settings, but because they represent three relevant interaction principles: adaptive two-stream fusion, auxiliary-feature-guided interaction, and relation-structured coupling. For fairness, all are implemented as lightweight interaction rules within the same dual-branch backbone, with the scan branches, encoder-decoder, optimizer, and training protocol fixed. Specifically, BFTT3D-style uses adaptive two-branch weighting, I2P-MAE-style uses one branch to guide the other before aggregation, and DTRG-style inserts a compact relation-aware coupling block before fusion. Thus, only the branch interaction rule differs, while all other components remain unchanged.

Element-wise addition gives the weakest result, with an average Dice of 80.28\%. The three style-adapted baselines improve this to 89.54\%, 89.50\%, and 89.46\%, respectively, while HSIC Gate performs best at 90.12\%. The organ-wise results show the same trend. On the gallbladder, Dice increases from 62.47\% to 78.96\%; on the pancreas, from 69.28\% to 81.41\%. HSIC Gate also achieves the best result on every reported organ in Table~\ref{tab:fusion_results_acm}. These results support Section~\ref{sec:method}: the effectiveness of the dual-branch design depends not only on branch diversity, but also on a compact fusion rule that adapts to the statistical dependence between branches.

\begin{table}[t]
  \caption{HSIC Gate sensitivity on Synapse under 30-epoch scratch training; hyperparameters are selected on the validation split only.}
  \label{tab:hsic_sensitivity_acm}
  \centering
  \footnotesize
  % \scriptsize
  \setlength{\tabcolsep}{5pt}
  \begin{tabular}{cccccc}
    \toprule
    \multicolumn{6}{c}{$T$ sensitivity with $k=64$} \\
    \midrule
    $T$ & 0.5 & 1.0 & 1.5 & 2.0 & 2.5 \\
    DSC & 68.14 & 69.57 & \textbf{71.44} & 68.54 & 67.64 \\
    \midrule
    \multicolumn{6}{c}{$k$ sensitivity with $T=1.5$} \\
    \midrule
    $k$ & 16 & 32 & 64 & 128 & -- \\
    DSC & 64.42 & 67.51 & \textbf{71.44} & 69.04 & -- \\
    \bottomrule
  \end{tabular}
\end{table}

\begin{table}[t]
  \caption{Measured ScanCache efficiency under fixed and dynamic input resolutions.}
  \label{tab:scancache_acm_measured}
  \centering
  % \footnotesize
  \scriptsize
  \setlength{\tabcolsep}{2pt}
  \begin{tabular}{llccc}
    \toprule
    Setting & Input / Mix & Latency & Gain / Hit & FPS \\
    \midrule
    Dual-branch & $256^2$ & $47.3 \rightarrow 31.9$ ms & 32.6\% / 99.07\% hit & 31.3 \\
    Dual-branch & $512^2$ & $72.1 \rightarrow 37.6$ ms & 47.9\% / 99.07\% hit & 26.6 \\
    TopoA-Scan only & $512^2$ & $46.1 \rightarrow 9.7$ ms & 79.0\% / 99.07\% hit & 103.2 \\
    \midrule
    Fixed & 100\% $512^2$ & $72.7 \rightarrow 37.9$ ms & 47.8\% / 99.07\% hit & 26.4 \\
    Two-scale & 50\% $256^2$ + 50\% $512^2$ & $60.4 \rightarrow 35.2$ ms & 41.7\% / 98.15\% hit & 28.4 \\
    Multi-scale & 5 recurring sizes & $59.5 \rightarrow 34.7$ ms & 41.7\% / 95.37\% hit & 28.8 \\
    Unique-scale/sample & all unique & $55.4 \rightarrow 37.5$ ms & 32.4\% / 86.48\% hit & 26.7 \\
    \bottomrule
  \end{tabular}
\end{table}

% \begin{table}[t]
%   % The top block reports end-to-end latency under fixed resolutions. The bottom block reports dynamic-resolution latency and cache hit rate over all ETMS cache queries in a full forward pass.
%   \caption{Compact ScanCache summary. }
%   \label{tab:scancache_acm}
%   \centering
%   \footnotesize
%   % \scriptsize
%   \setlength{\tabcolsep}{5pt}
%   \begin{tabular}{llccc}
%     \toprule
%     Setting & Input / Mix & Latency & Gain / Hit & FPS \\
%     \midrule
%     Dual-branch & $256^2$ & $26.2 \rightarrow 15.1$ ms & 42.4\% & 66.2 \\
%     Dual-branch & $512^2$ & $53.4 \rightarrow 16.2$ ms & 69.7\% & 61.7 \\
%     TopoA-Scan only & $512^2$ & $46.7 \rightarrow 10.0$ ms & 78.6\% & 100.0 \\
%     \midrule
%     Fixed & 100\% $512^2$ & 23.4 ms & 99.67\% hit & -- \\
%     Two-scale & 50\% $256^2$ + 50\% $512^2$ & 19.5 ms & 99.67\% hit & -- \\
%     Multi-scale & 5 recurring sizes & 30.3 ms & 99.00\% hit & -- \\
%     Unique-scale/sample & all unique & 38.6 ms & 85.33\% hit & -- \\
%     \bottomrule
%   \end{tabular}
% \end{table}

\begin{figure*}[ht]
\centerline
{\includegraphics[width=0.85\textwidth]{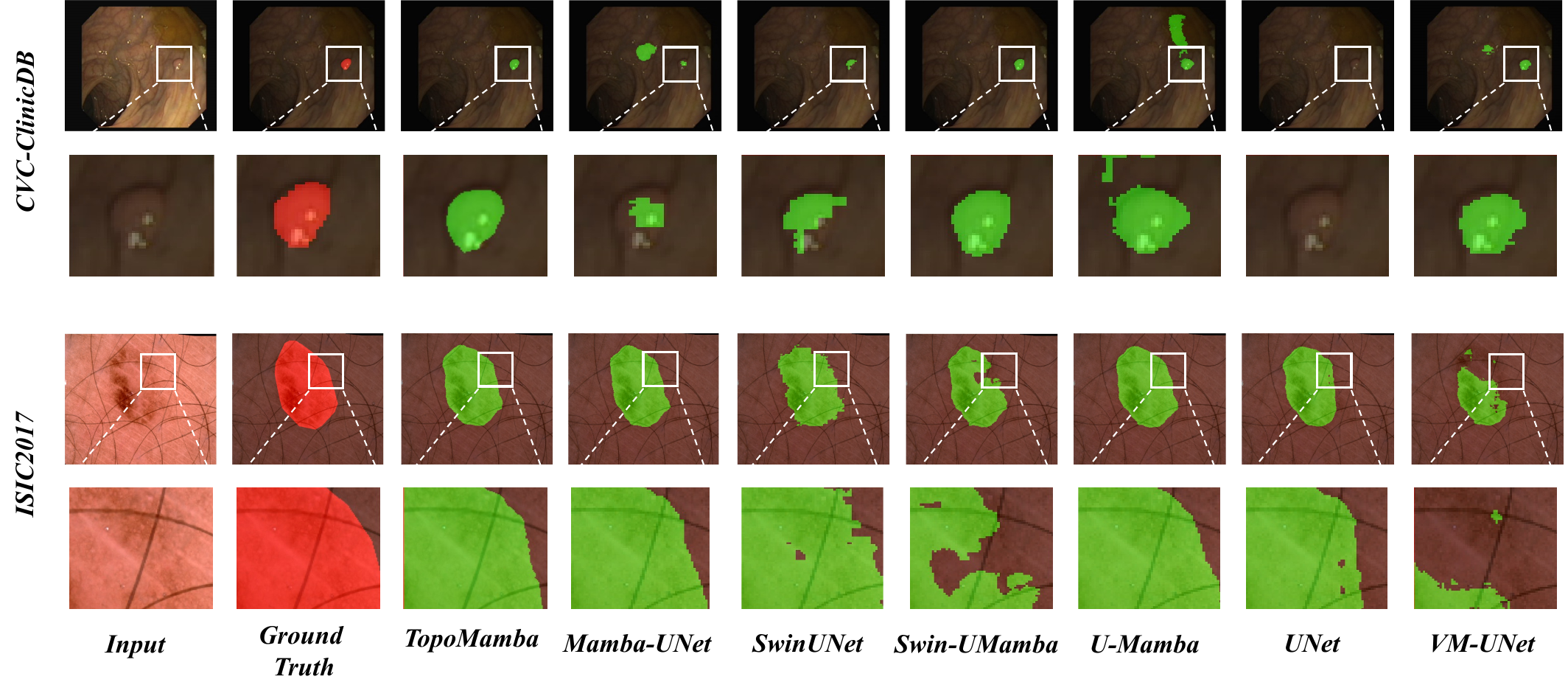}}
\Description{Qualitative comparisons on dermoscopy and endoscopy benchmarks showing that TopoMamba produces cleaner boundaries and fewer false positives than representative baselines.}
\caption{Representative ISIC 2017 and CVC-ClinicDB test cases (fixed random seed). Additional random cases are provided in the supplementary material.}
\label{fig:ISIC_CVC_vis}
\end{figure*}

% \begin{figure*}[ht]
% \centering
% \begingroup
% \setlength{\unitlength}{0.0098\linewidth}
% \begin{picture}(100,45.38)
% \put(0,0){\includegraphics[width=0.98\linewidth]{figs/ISIC_CVC_Vis.pdf}}
% \put(28.0,8.5){\makebox(0,0)[c]{\colorbox{gray!25}{\rule{0pt}{8.5mm}\hspace{22mm}}}}
% \put(28.0,8.5){\makebox(0,0)[c]{\colorbox{white}{\scriptsize\shortstack[c]{\textit{TopoMamba}\\\textit{Tiny}}}}}
% \end{picture}
% \endgroup
% \Description{Qualitative comparisons on dermoscopy and endoscopy benchmarks showing that TopoMamba produces cleaner boundaries and fewer false positives than representative baselines.}
% \caption{Segmentation results on ISIC2017  and CVC-ClinicDB datasets.}
% \label{fig:ISIC_CVC_vis}
% \end{figure*}

\subsection{HSIC Sensitivity}

Table~\ref{tab:hsic_sensitivity_acm} studies two HSIC Gate hyperparameters, namely the temperature $T$ and the projection size $k$, under 30-epoch scratch training. These hyperparameters are selected on the Synapse validation split only and then fixed for the main experiments. The absolute numbers are lower than those in Table~\ref{tab:synapse_results_acm} because this study uses short scratch training rather than the full training recipe.

Experimental results indicates that both overly sharp and overly smooth gating are suboptimal, and that excessively small or large projection dimensions are also unfavorable. Based on validation performance, we therefore adopt $T=1.5$ and $k=64$ in all main experiments.

% For temperature, Dice increases from 68.14 at $T=0.5$ to 69.57 at $T=1.0$, peaks at 71.44 at $T=1.5$, and then decreases to 68.54 and 67.64 at $T=2.0$ and $2.5$, respectively. For projection size, the corresponding scores are 64.42, 67.51, 71.44, and 69.04 for $k=16,32,64,128$. This pattern indicates that both overly sharp and overly smooth gating are suboptimal, and that excessively small or large projection dimensions are also unfavorable. Based on validation performance, we therefore adopt $T=1.5$ and $k=64$ in all main experiments.

\subsection{ScanCache Efficiency and Deployment}
\label{sec:scancache_efficiency}

Table~\ref{tab:scancache_acm_measured} summarizes the measured efficiency of ScanCache from two perspectives on 100 randomly sampled Synapse slices, matching the released measurement artifact. The upper block reports end-to-end latency under fixed input resolutions, with and without caching. For the dual-branch model at $256^2$, latency is reduced from 47.3 ms to 31.9 ms, corresponding to a 32.6\% reduction and 31.3 FPS. At $512^2$, latency drops from 72.1 ms to 37.6 ms, a 47.9\% reduction, yielding 26.6 FPS. For a single TopoA-Scan branch at $512^2$, latency decreases from 46.1 ms to 9.7 ms, corresponding to a 79.0\% reduction and 103.2 FPS. These measured reductions are practically important because TopoA-Scan relies on explicit scan indices, and index-service overhead therefore contributes directly to wall-clock runtime.

The lower block evaluates cache behavior under dynamic input resolutions. The cache hit rate is 99.07\% for fixed $512^2$, remains 98.15\% for a 50/50 mixture of $256^2$ and $512^2$, and is still 95.37\% for a five-scale recurring mixture. Even when every external input size is unique, the hit rate remains 86.48\%, with latency reduced from 55.4 ms to 37.5 ms. This is because repeated internal feature-map sizes still recur within a complete U-shaped forward pass, enabling effective index reuse despite variation in external resolution. Overall, ScanCache is not merely an implementation detail; it is a deployment-oriented mechanism that makes explicit topology-aware scanning substantially more practical in real segmentation pipelines.

\section{Visualization Analysis}
We provide qualitative results in Figures~\ref{fig:synapse_vis}, \ref{fig:ISIC_CVC_vis}, and \ref{fig:attention_heatmap} to complement the quantitative analysis. Figures~\ref{fig:synapse_vis} and \ref{fig:ISIC_CVC_vis} are sampled from the test set with a fixed random seed, with more random cases in the supplementary material. In Synapse, TopoMamba better preserves thin or oblique anatomy and reduces incomplete structures, local misclassification, and fragmented boundaries relative to strong baselines. In ISIC2017 and CVC-ClinicDB, it more reliably traces irregular lesion and polyp boundaries while suppressing distractors such as hair, illumination variation, and specular highlights, consistent with the specificity gains in Table~\ref{tab:comparison_results_acm}. Figure~\ref{fig:attention_heatmap} is treated as qualitative support only; the controlled ablations and topology-oriented analysis provide the main quantitative evidence for branch complementarity.

\section{Discussion}
\label{sec:discussion}

TopoMamba combines a topology-aware scan prior, dependence-aware fusion, and deployment-minded caching in a single SSM framework. Its main value for multimedia research is that the same scan-and-fuse design remains effective across heterogeneous medical visual media while also supporting a volumetric TopoMamba-3D result. We deliberately prioritize a compact fusion rule and efficient scan reuse over a larger multimodal alignment head or report-conditioned prompting stack. That choice keeps the method easier to analyze and deploy. More broadly, we expect the decomposition to remain useful beyond the exact datasets studied here. TopoA-style scan ordering injects a reusable non-axial structural bias, HSIC Gate offers a lightweight alternative to heavier fusion heads, and ScanCache promotes deployment details to architectural design choices.

\textbf{Limitations.}
The proposed architecture is designed to effectively exploit non-axial structural continuity through topology-aware scanning and lightweight dependence-aware fusion. While it achieves strong performance and favorable efficiency for heterogeneous medical segmentation, its applicability to settings that require explicit topology guarantees has not been validated in this work.

% All datasets used in this work are public research benchmarks and were used under their published research terms. The proposed model is intended for research evaluation only and is not clinically validated for diagnostic deployment; real-world use would require prospective validation, site-specific calibration, and clinician oversight.

\clearpage
\bibliographystyle{IEEEtran}
\bibliography{main}

\clearpage
% \clearpage
\setcounter{page}{1}
\appendix
\setcounter{figure}{0}
\setcounter{table}{0}
\setcounter{algorithm}{0}
\renewcommand{\thefigure}{S\arabic{figure}}
\renewcommand{\thetable}{S\arabic{table}}
\renewcommand{\thealgorithm}{S\arabic{algorithm}}

\section{Additional Method Details}

\subsection{Preliminaries of SSM}
State-space models (SSMs) describe sequential processing through a hidden-state evolution:
\begin{align}
\frac{d\mathbf{h}(t)}{dt} &= \mathbf{A}\mathbf{h}(t) + \mathbf{B}\mathbf{x}(t), \\
\mathbf{y}(t) &= \mathbf{C}\mathbf{h}(t) + \mathbf{D}\mathbf{x}(t),
\end{align}
where $\mathbf{h}(t)$, $\mathbf{x}(t)$, and $\mathbf{y}(t)$ denote the hidden state, input, and output, respectively. After discretization with step size $\Delta$, the recurrence becomes
\begin{align}
\mathbf{h}_k &= \overline{\mathbf{A}}\mathbf{h}_{k-1} + \overline{\mathbf{B}}\mathbf{x}_k, \\
\mathbf{y}_k &= \mathbf{C}\mathbf{h}_k + \mathbf{D}\mathbf{x}_k.
\end{align}
Vision-Mamba-style models serialize a 2D or 3D feature tensor into several scan sequences and then apply selective scan with input-dependent parameters. TopoMamba keeps the pretrained-compatible Cross-Scan branch and adds a complementary TopoA-Scan branch inside Efficient Topological Multi-Scan (ETMS).

\subsection{Cross-Scan Formulation}
Given a feature map $\mathbf{X} \in \mathbb{R}^{B \times C \times H \times W}$ with $L=HW$, Cross-Scan produces four axis-aligned traversal orders that correspond to row-major, column-major, and their reversals~\cite{105VMamba2024vmamba}. The resulting serialized tensor can be written as
\begin{align}
\mathbf{X}_{\text{cross}} &= \operatorname{CrossScan}(\mathbf{X}) \in \mathbb{R}^{B \times 4 \times C \times L}, \\
\mathbf{F}_{\text{cross}} &= \operatorname{SelectiveScan}(\mathbf{X}_{\text{cross}}, \boldsymbol{\Delta}, \mathbf{A}, \mathbf{B}, \mathbf{C}),
\end{align}
where the selective parameters are omitted for brevity.
This branch preserves compatibility with VMamba-style pretrained weights and remains sensitive to axis-aligned structures.

\subsection{Full Definition of TopoA-Scan}
For a feature map of size $(H, W)$ with $L=HW$, the topology-aware branch traverses both main diagonals and anti-diagonals. For each diagonal index $s \in \{0, \dots, H+W-2\}$, define
\begin{equation}
i_{\min}(s)=\max(0,\, s-(W-1)), \qquad i_{\max}(s)=\min(s,\, H-1).
\end{equation}
We first define the ordered coordinate lists
\begin{equation}
\mathbf{d}^{(s)} = \big[(i,\, s-i)\big]_{i=i_{\min}(s)}^{i_{\max}(s)},
\end{equation}
and
\begin{equation}
\mathbf{a}^{(s)} = \big[(i,\, W-1-(s-i))\big]_{i=i_{\min}(s)}^{i_{\max}(s)},
\end{equation}
where entries are ordered by increasing $i$. The alternating TopoA segment orders are then
\begin{equation}
\operatorname{diag}(s)=
\begin{cases}
\mathbf{d}^{(s)}, & s \text{ even},\\
\operatorname{reverse}(\mathbf{d}^{(s)}), & s \text{ odd},
\end{cases}
\end{equation}
and
\begin{equation}
\operatorname{antidiag}(s)=
\begin{cases}
\mathbf{a}^{(s)}, & s \text{ even},\\
\operatorname{reverse}(\mathbf{a}^{(s)}), & s \text{ odd}.
\end{cases}
\end{equation}
Concatenating all segments gives two full coordinate sequences:
\begin{equation}
\mathbf{p}_{\text{diag}}=\operatorname{concat}_{s=0}^{H+W-2}\operatorname{diag}(s), \qquad
\mathbf{p}_{\text{antidiag}}=\operatorname{concat}_{s=0}^{H+W-2}\operatorname{antidiag}(s).
\end{equation}
After flattening each coordinate $(i,j)$ into the row-major index $iW+j$, we obtain the integer index sequences $\mathbf{i}_{\text{diag}}, \mathbf{i}_{\text{antidiag}} \in \mathbb{Z}^{L}$ and construct the complete forward index tensor
\begin{equation}
\mathbf{I}_{\text{fwd}}=
\begin{bmatrix}
\mathbf{i}_{\text{diag}} \\
\mathbf{i}_{\text{antidiag}} \\
\operatorname{flip}(\mathbf{i}_{\text{diag}}) \\
\operatorname{flip}(\mathbf{i}_{\text{antidiag}})
\end{bmatrix}
\in \mathbb{Z}^{4 \times L}.
\end{equation}
Its inverse tensor is defined by
\begin{equation}
\mathbf{I}_{\text{inv}}[k,\mathbf{I}_{\text{fwd}}[k,j]] = j,\qquad k \in [0,3],\ j \in [0,L-1].
\end{equation}
Thus, features are first gathered by $\mathbf{I}_{\text{fwd}}$ for selective scan and later restored to raster order by $\mathbf{I}_{\text{inv}}$.

\begin{figure*}[t]
\centering
\includegraphics[width=0.98\linewidth]{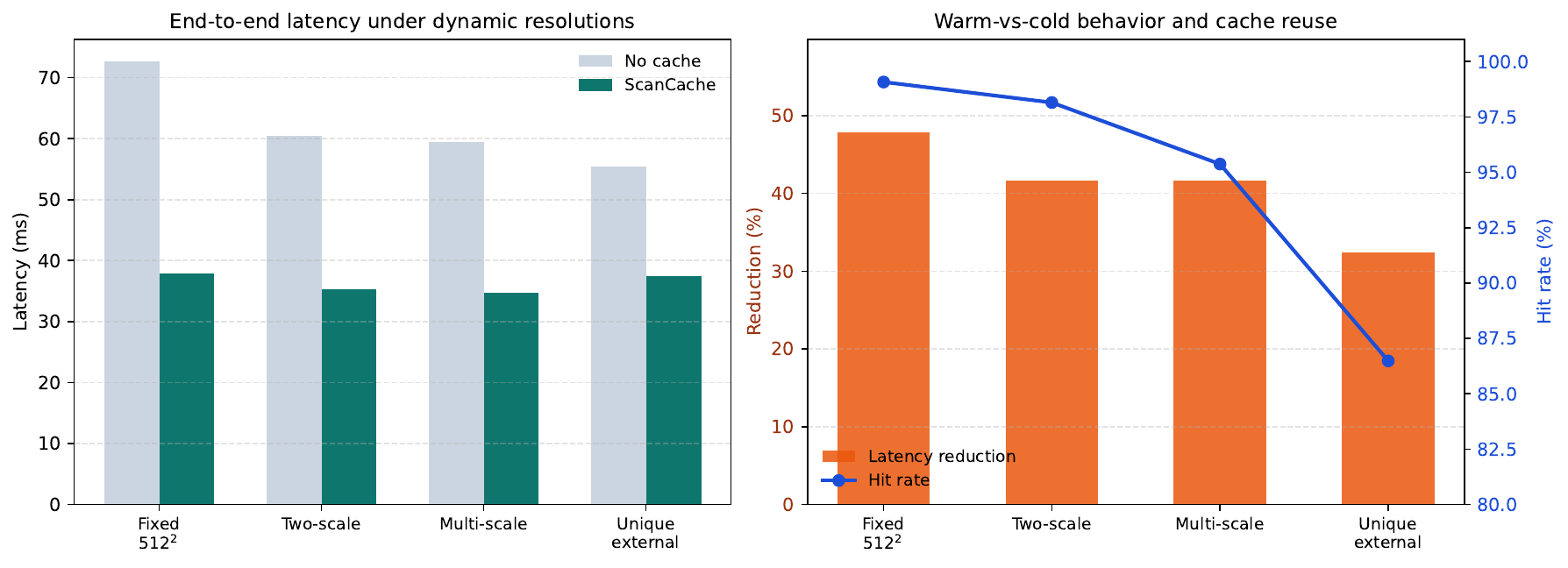}
\caption{Warm-vs-cold ScanCache behavior under dynamic input resolutions. The left panel compares end-to-end latency with and without caching. The right panel shows latency reduction together with cache hit rate for the same scenarios.}
\label{fig:scancache_dynamic_breakdown}
\end{figure*}

\begin{figure*}[ht]
\centering
\centerline{\includegraphics[width=0.97\textwidth]{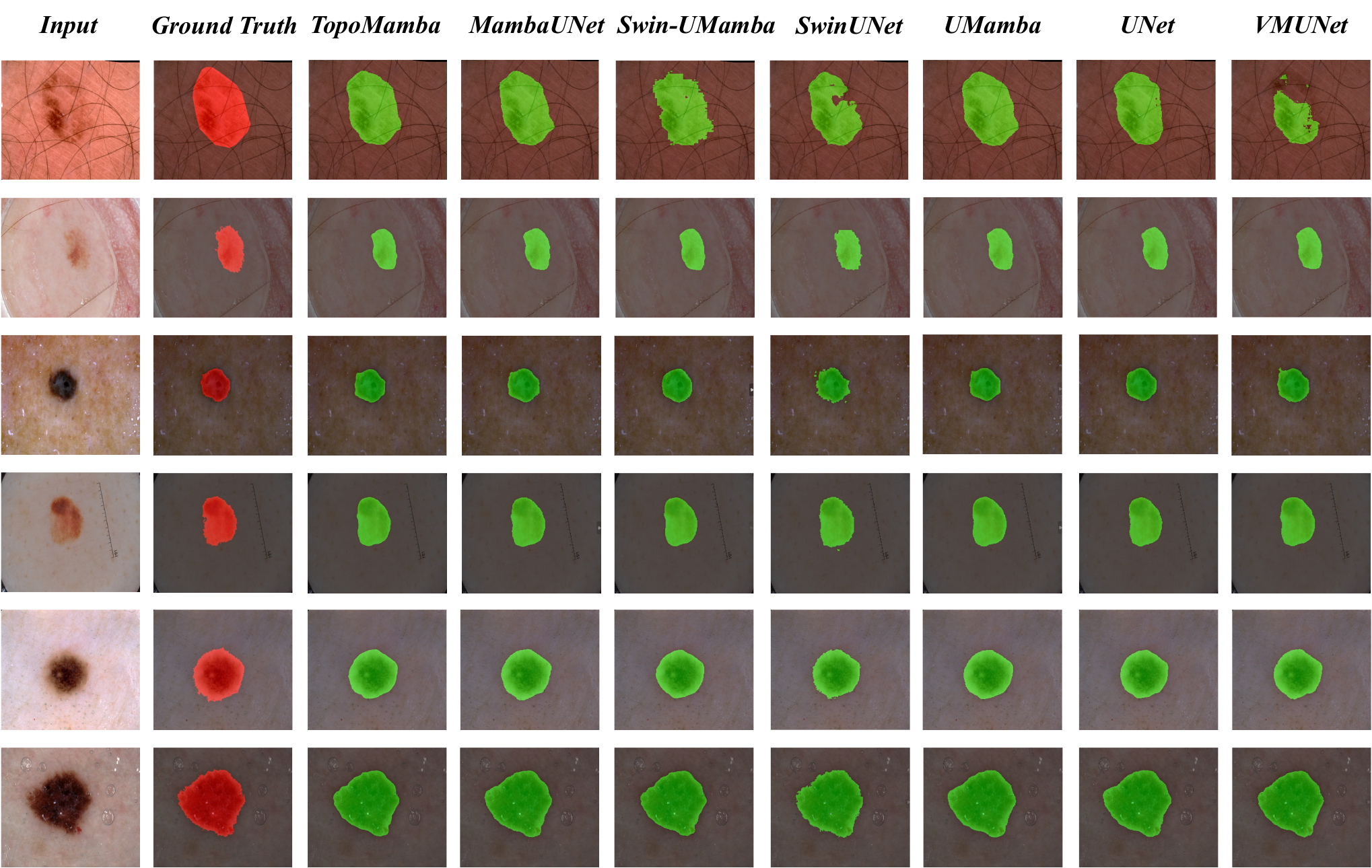}}
\Description{Additional qualitative dermoscopy examples showing that TopoMamba recovers lesion extent and boundaries more reliably than representative baselines.}
\caption{More Segmentation results on ISIC2017 dataset.}
\label{fig:ISIC_more_vis}
\end{figure*}

\begin{figure*}[ht]
\centering
\centerline{\includegraphics[width=0.97\textwidth]{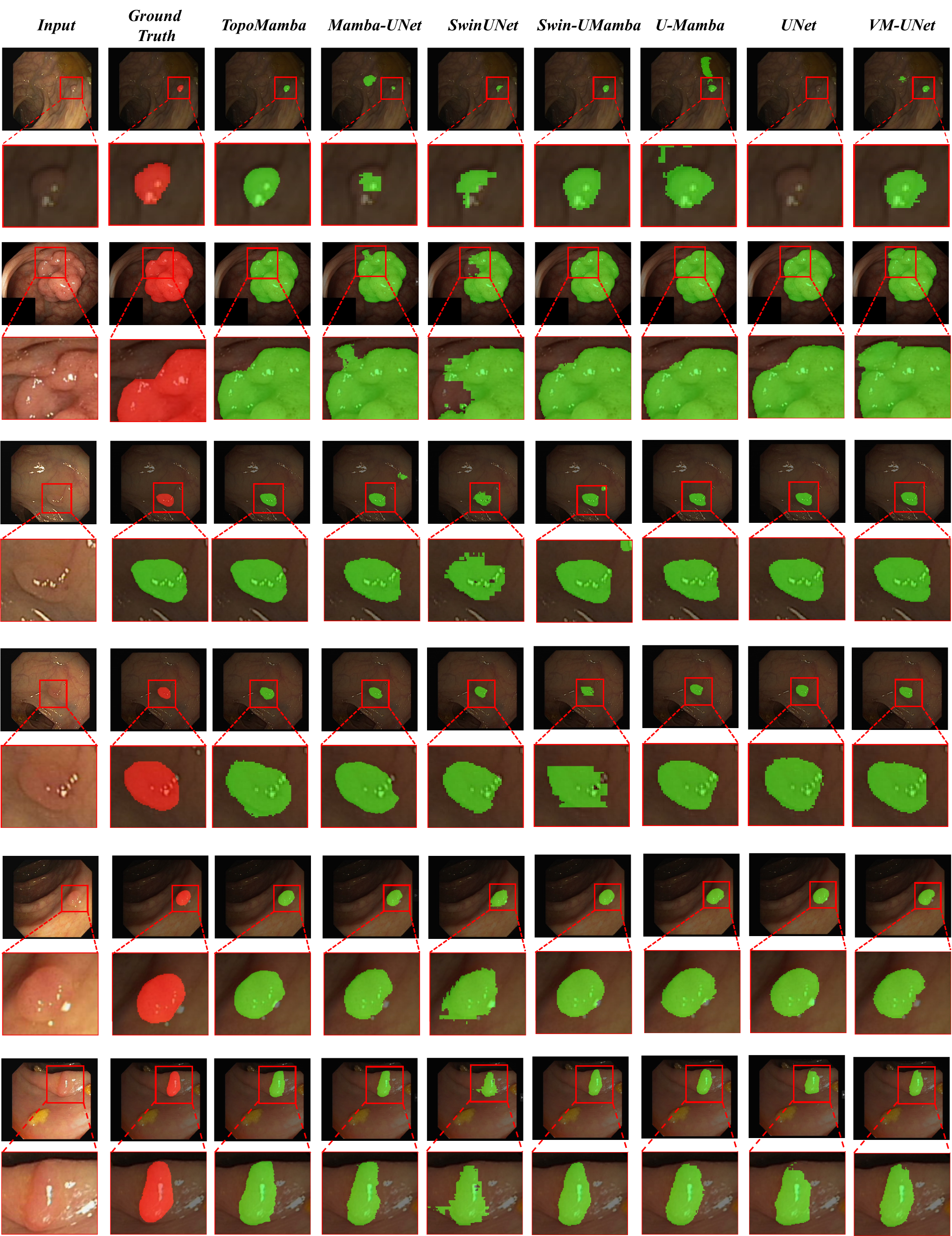}}
\Description{Additional qualitative endoscopy examples showing that TopoMamba produces cleaner polyp boundaries and fewer false positives than representative baselines.}
\caption{More Segmentation results on CVC-ClinicDB dataset.}
\label{fig:CVC_more_vis}
\end{figure*}

\subsection{Theoretical Properties of TopoA-Scan}

\begin{proposition}[Spatial locality of the base diagonal TopoA sequence]
Let $\mathcal{S}$ denote the full diagonal TopoA order obtained by concatenating all $\operatorname{diag}(s)$ segments. For any consecutive points $\mathbf{p}_m=(x_m,y_m)$ and $\mathbf{p}_{m+1}=(x_{m+1},y_{m+1})$ in $\mathcal{S}$,
\begin{equation}
\|\mathbf{p}_{m+1}-\mathbf{p}_m\|_2 \in \{1,\sqrt{2}\},
\end{equation}
and therefore $\|\mathbf{p}_{m+1}-\mathbf{p}_m\|_2 \le \sqrt{2}$.
\end{proposition}

\begin{proof}
Inside a single diagonal segment, consecutive points differ by $(1,-1)$ when the traversal direction is forward and by $(-1,1)$ when the traversal direction is reversed, so the Euclidean distance is always $\sqrt{2}$. At the boundary between two consecutive diagonal segments, the alternating reversal ensures that the terminal point of one segment and the initial point of the next segment differ by either $(1,0)$ or $(0,1)$, hence they are 4-neighbors and their distance is $1$. These are the only two cases.
\end{proof}

\begin{corollary}[Extension to anti-diagonal and reversed TopoA sequences]
The same $\sqrt{2}$ locality bound holds for the anti-diagonal order and for both reversed orders in $\mathbf{I}_{\text{fwd}}$.
\end{corollary}

\begin{proof}
Replacing $j$ with $W-1-j$ maps main diagonals to anti-diagonals without changing pairwise Euclidean distances. Reversing a finite sequence preserves the set of adjacent pair distances. Hence all four TopoA directions share the same locality bound.
\end{proof}

\begin{proposition}[Permutation property and bijectivity]
Each row of $\mathbf{I}_{\text{fwd}}$ is a permutation of $\{0,\dots,L-1\}$; equivalently, each row defines a bijection between scan positions and raster indices.
\end{proposition}

\begin{proof}
For the main-diagonal row, every spatial location $(i,j)$ belongs to exactly one diagonal indexed by $s=i+j$, and within that diagonal it receives exactly one local position. Concatenating all diagonal segments therefore visits each spatial location exactly once. The same argument applies to the anti-diagonal row, where the relevant segment index is $s=i+(W-1-j)$. The remaining two rows are obtained by applying \texttt{flip} to already bijective sequences, which preserves the permutation property. Hence every row of $\mathbf{I}_{\text{fwd}}$ is a permutation of $\{0,\dots,L-1\}$.
\end{proof}

\begin{proposition}[Correctness of $(\mathbf{I}_{\text{fwd}}, \mathbf{I}_{\text{inv}})$]
Let $\mathbf{z} \in \mathbb{R}^{L}$ be any flattened feature vector and let $\mathbf{z}_{k}^{\text{scan}} = \mathbf{z}[\mathbf{I}_{\text{fwd}}[k,:]]$. Then $\mathbf{z}_{k}^{\text{scan}}[\mathbf{I}_{\text{inv}}[k,:]] = \mathbf{z}$ for every $k \in [0,3]$.
\end{proposition}

\begin{proof}
Fix any $k$ and any raster index $\ell \in \{0,\dots,L-1\}$. Let $j=\mathbf{I}_{\text{inv}}[k,\ell]$. By construction of the inverse index, $\mathbf{I}_{\text{fwd}}[k,j]=\ell$. Therefore,
\begin{equation}
\mathbf{z}_{k}^{\text{scan}}[\mathbf{I}_{\text{inv}}[k,\ell]]
=
\mathbf{z}_{k}^{\text{scan}}[j]
=
\mathbf{z}[\mathbf{I}_{\text{fwd}}[k,j]]
=
\mathbf{z}[\ell].
\end{equation}
Since this holds for every $\ell$, exact restoration follows.
\end{proof}

\subsection{HSIC Gate: Implementation and Calibration}

The HSIC Gate keeps only a lightweight scalar calibration parameter while comparing the statistics of $\mathbf{F}_{\text{cross}}$ and $\mathbf{F}_{\text{TopoA}}$. The implementation follows four steps:
\begin{enumerate}
    \item Apply Johnson-Lindenstrauss projection~\cite{150HSIC_projection_johnson1984extensions,177JL_Proj2003database} to reduce the sequence dimension before kernel construction.
    \item Build RBF kernels~\cite{178RBF_scholkopf2002learning} with a median bandwidth heuristic on the projected channel descriptors.
    \item Center the kernels by row mean, column mean, and global mean before computing the normalized Frobenius inner product.
    \item Convert the resulting HSIC score into a sigmoid gate and retain a TopoA-biased residual shortcut for stability.
\end{enumerate}
We intentionally avoid stronger claims such as target-variable dependence, mutual-information approximation, or topology guarantees, because the gate is used here purely as a compact dependence-aware fusion rule.
In the paper setting, $\alpha$ is initialized to $0.5$, the temperature is fixed to $T=1.5$, and the projection cap is set to $k=64$ after validation-only sensitivity analysis.

\subsection{ScanCache: Implementation and Amortized Analysis}

Figure~\ref{fig:scancache_dynamic_breakdown} complements the deployment analysis in the main paper by visualizing ScanCache behavior under recurring and dynamic input resolutions. The key message is that the measured gain does not come from changing the selective-scan computation itself, but from amortizing explicit scan-index service across repeated internal feature-map sizes. Even when the external image size varies, the internal resolutions inside a complete U-shaped forward pass still recur, so cached topology-aware scan indices can be reused effectively. This is why ScanCache should be interpreted as a deployment-oriented systems component that improves the practical usability of explicit TopoA-Scan rather than as a change to the segmentation objective or backbone capacity.

\section{Additional Qualitative Results}

These pages extend the fixed-random-seed qualitative protocol of the main paper with additional random test cases from ISIC 2017 and CVC-ClinicDB. Consistent with the quantitative results, TopoMamba more reliably follows irregular lesion and polyp boundaries while reducing distractor responses caused by hair occlusion, illumination variation, specular highlights, and surrounding textures. Because dermoscopy and endoscopy present different appearance statistics and different types of false positives, these additional pages provide a direct visual check that the same scan-and-fuse design transfers across heterogeneous medical visual media. These qualitative examples are intended as supporting evidence only; the main claims about cross-media generalization remain grounded in the controlled quantitative results and analysis reported in the paper.

\section{Experimental Details and Reproducibility}

\subsection{Datasets, Splits, and Preprocessing}
Our implementation follows the following data conventions.
\begin{itemize}
    \item \textbf{Synapse.} We follow the standard 18/12 patient split used in the main paper and report Dice and HD95 on held-out volumes. For the 2D implementation, the training set contains 3779 preprocessed slices and the test set contains 12 held-out volumes. During training, random rotation/flip or random rotation is applied, slices are resized to the target input size, and single-channel CT slices are repeated to three channels when pretrained 2D backbones are used.
    \item \textbf{ISIC2017.} We follow the standard \texttt{train/val/test} split. The preprocessing pipeline applies random rotation/flip, random rotation, color jitter, resize to $512^2$, and normalization to $[0,1]$.
    \item \textbf{CVC-ClinicDB.} We follow a 5-fold cross-validation protocol with \texttt{random\_state=42}. The preprocessing pipeline applies random rotation/flip, random rotation, color jitter, resize to $512^2$, and normalization.
\end{itemize}
Evaluation follows the main paper: Dice and HD95 are reported on Synapse, while DSC, mIoU, and specificity are reported on ISIC2017 and CVC-ClinicDB. Unless otherwise stated, the main 2D experiments follow the paper protocol with AdamW, cosine decay, three random seeds, and input size $512^2$.

\subsection{Latency Measurement Protocol}
To match Table~7 in the main paper, latency is measured on 100 randomly sampled Synapse test slices with batch size $1$ on one NVIDIA GeForce RTX 4090 (24 GB), Python 3.11.5, PyTorch 2.1.2+cu121, and CUDA 12.1. We perform 8 warm-up forwards and then time one forward pass per sample. The measured latency includes the per-sample host-to-device tensor transfer inside the evaluation loop, but excludes disk I/O and dataloader startup. The dynamic-resolution evaluation follows four scenarios: fixed $512^2$, a 50/50 mixture of $256^2$ and $512^2$, a five-scale recurring mixture, and a unique external-size rule \texttt{side\_i = 256 + 2i}. This protocol is reported separately from the amortized proposition in Section A.6: the proposition concerns index service alone, whereas the measured latency also includes reorder, selective scan, decoder computation, and transfer overhead.

\end{document}